\RequirePackage{fix-cm}
\documentclass[twocolumn]{svjour3_arxiv}          % twocolumn
\pdfoutput=1
\smartqed  % flush right qed marks, e.g. at end of proof
\usepackage{graphicx}
\usepackage[cmex10]{amsmath}
\usepackage[tight,footnotesize]{subfigure}
\usepackage{natbib}
\usepackage{times}
\usepackage{color}
\newcommand{\bbm}{\begin{bmatrix}}
\newcommand{\ebm}{\end{bmatrix}}
\DeclareMathAlphabet{\mbfa}{OT1}{ptm}{b}{n}
\newcommand{\mbf}[1]{\mbfa{#1}} % stops warnings in texstudio
\newcommand{\mbs}[1]{{\boldsymbol{#1}}}

\newcommand{\diag}{\textup{diag}}

\DeclareMathOperator{\tr}{tr}

\newcommand{\mbsdot}[1]{{\dot{\boldsymbol{#1}}}}

\newcommand{\mbfdel}[1]{{\delta{\mbf{#1}}}}
\newcommand{\mbfdot}[1]{{\dot{\mbf{#1}}}}

\newcommand{\mbfhat}[1]{{\hat{\mbf{#1}}}}
\newcommand{\mbfcheck}[1]{{\check{\mbf{#1}}}}

\newcommand{\primean}{\mbfcheck{x}}
\newcommand{\pricov}{\mbfcheck{P}}
\newcommand{\opmean}{\mbf{x}_{\text{op}}}
\newcommand{\postmean}{\mbfhat{x}}
\newcommand{\postcov}{\mbfhat{P}}
\newcommand{\sdemat}{\mbf{F}}
\newcommand{\ik}{n}
\newcommand{\iK}{N}
\newcommand{\ikf}{k}
\newcommand{\iKf}{K}

\newcommand{\iQry}{J}
\newcommand{\ihy}{w}
\newcommand{\iHy}{W}
\newenvironment{tolxemerg}[2]{%
	\par
	\ifx\empty#1\empty\else\tolerance=#1\relax\fi
	\ifx\empty#2\empty\else\emergencystretch=#2\relax\fi
}{%
\par
}
\journalname{Autonomous Robots}

\begin{document}

\title{Batch Nonlinear Continuous-Time Trajectory Estimation as Exactly Sparse Gaussian Process Regression}

\author{Sean Anderson         \and
        Timothy D. Barfoot    \and
        Chi Hay Tong          \and
        Simo S\"{a}rkk\"{a}
}

\institute{S. Anderson \at
              Autonomous Space Robotics Lab, University of Toronto Institute for Aerospace Studies, Canada \\
              \email{sean.anderson@mail.utoronto.ca}
           \and
           T. D. Barfoot \at
	          Autonomous Space Robotics Lab, University of Toronto Institute for Aerospace Studies, Canada \\
              \email{tim.barfoot@utoronto.ca}
           \and
           C. H. Tong \at
              Mobile Robotics Group, University of Oxford, United Kingdom \\
              \email{chi@robots.ox.ac.uk}
           \and
           S. S\"{a}rkk\"{a} \at
              Department of Biomedical Engineering and Computational Science, Aalto University, Finland \\
              \email{simo.sarkka@aalto.fi}
}

\date{Received: 20 November 2014}
%\date{Received: date / Accepted: date}
% The correct dates will be entered by the editor

\maketitle

\begin{abstract}
\begin{tolxemerg}{}{2pt}
In this paper, we revisit batch state estimation through the lens of {\em Gaussian process} (GP) regression.  We consider continuous-discrete estimation problems wherein a trajectory is viewed as a one-dimensional GP, with time as the independent variable.  Our continuous-time prior can be defined by any nonlinear, time-varying stochastic differential equation driven by white noise; this allows the possibility of smoothing our trajectory estimates using a variety of vehicle dynamics models (e.g., `constant-velocity'). We show that this class of prior results in an inverse kernel matrix (i.e., covariance matrix between all pairs of measurement times) that is exactly sparse (block-tridiagonal) and that this can be exploited to carry out GP regression (and interpolation) very efficiently.  
When the prior is based on a linear, time-varying stochastic differential equation and the measurement model is also linear, this GP approach is equivalent to classical, discrete-time smoothing (at the measurement times); when a nonlinearity is present, we iterate over the whole trajectory to maximize accuracy.
We test the approach experimentally on a simultaneous trajectory estimation and mapping problem using a mobile robot dataset.
\end{tolxemerg}
\end{abstract}

\section{Introduction}
\label{sec:introduction}

Probabilistic state estimation has been a core topic in mobile robotics since the 1980s \citep{durrantwhyte88,smith86,smith90}, often as part of  the {\em simultaneous localization and mapping} (SLAM) problem \citep{durrantwhyte06b,durrantwhyte06a}.
Early work in estimation theory focused on recursive (as opposed to batch) formulations \citep{kalman60}, and this was mirrored in the formulation of SLAM as a filtering problem \citep{smith90}.  However, despite the fact that continuous-time estimation techniques have been available since the 1960s \citep{jazwinski70,kalman61}, trajectory estimation for mobile robots has been formulated almost exclusively in discrete time.  

\begin{tolxemerg}{}{2pt}
\citet{lu97} showed how to formulate SLAM as a batch estimation problem incorporating both odometry measurements (to smooth solutions) as well as landmark measurements.  This can be viewed as a generalization of {\em bundle adjustment} \citep{brown58,sibley10}, which did not incorporate odometry.  Today, batch approaches in mobile robotics are commonplace (e.g., GraphSLAM by \citet{thrun06}).  \citet{kaess08} show how batch solutions can be efficiently updated as new measurements are gathered and \citet{strasdat10} show that batch methods are able to achieve higher accuracy than their filtering counterparts, for the same computational cost.   Most of these results are formulated in discrete time.
\end{tolxemerg}

\begin{figure}
	\includegraphics[width=\columnwidth]{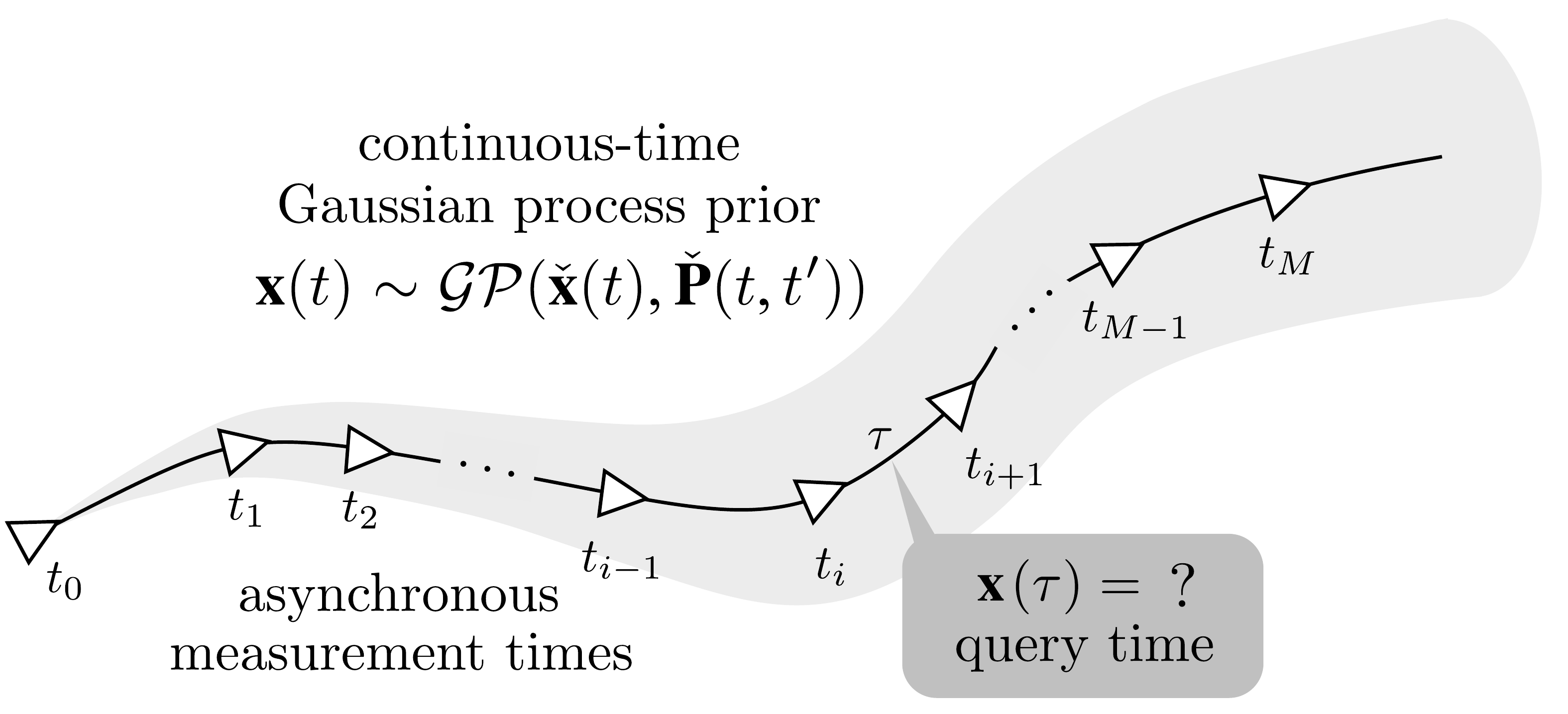}\\
	\caption{To carry out batch trajectory estimation, we use GP regression with a smooth, continuous-time prior and discrete-time measurements.  This allows us to query the trajectory at any time of interest, $\tau$.  }
	\label{fig:setup}
\end{figure}

\begin{tolxemerg}{}{2pt}
Discrete-time representations of robot trajectories are sufficient in many situations, but they do not work well when estimating motion from certain types of sensors (e.g., rolling-shutter cameras and scanning laser-rangefinders) and sensor combinations (e.g., high datarate, asynchronous).  In these cases, a smooth, continuous-time representation of the trajectory is more suitable. For example, in the case of estimating motion from a scanning-while-moving  sensor, a discrete-time approach (with no motion prior) can fail to find a unique solution; something is needed to  tie together the observations acquired at many unique timestamps.   Additional sensors (e.g., odometry or inertial measurements) could be introduced to serve in this role, but this may not always be possible. In these cases, a motion prior can be used instead (or as well), which is most naturally expressed in continuous time.
\end{tolxemerg}

\begin{tolxemerg}{}{2pt}
One approach to continuous-time trajectory representation is to use interpolation (e.g., linear, spline) directly between nearby discrete poses \citep{bibby10,bosse09,dong_fsr12,furgale_icra12,hedborg12,lovegrove13}.  Instead, we choose to represent the trajectory nonparametrically as a one-dimensional Gaussian process (GP) \citep{rasmussen06}, with time as the independent variable (see Figure~\ref{fig:setup}). \citet{tong_crv12,tong_ijrr13b} show that querying the state of the robot at a time of interest can be viewed as a nonlinear, GP regression problem.  While their approach is very general, allowing a variety of GP priors over robot trajectories, it is also quite expensive due to the need to invert a large, dense kernel matrix. 
\end{tolxemerg}

While GPs have been used in robotic state estimation to accomplish dimensionality reduction \citep{ferris06,ferris07,lawrence03} and to represent the measurement and motion models \citep{deisenroth12,ko09,ko11}, these uses are quite different than representing the latent robot trajectory as a GP \citep{tong_crv12,tong_ijrr13b}.

\begin{tolxemerg}{}{2pt}
In this paper, we consider a particular class of GPs, generated by nonlinear, time-varying (NTV) stochastic differential equations (SDE) driven by white noise.
We first show that GPs based on linear, time-varying (LTV) SDEs have an inverse kernel matrix that is {\em exactly sparse} (block-tridiagonal) and can be derived in closed form;
an approximation for GPs based on a NTV SDE is then shown that results in the same sparsity properties as the linear case.
Concentrating on this class of covariance functions results in only a minor loss of generality, because many commonly used covariance functions such the Mat\'{e}rn class and the squared exponential covariance function can be exactly or approximately represented as linear SDEs \citep{hartikainen10,sarkka13,solin14}.  We provide an example of this relationship at the end of this paper.  The resulting sparsity allows the approach of \citet{tong_crv12,tong_ijrr13b} to be implemented very efficiently.   The intuition behind why this is possible is that the state we are estimating is {\em Markovian} for this class of GPs, which implies that the corresponding precision matrices are sparse \citep{lindgren11}.
\end{tolxemerg}

\begin{tolxemerg}{300}{2pt}
This sparsity property has been exploited in estimation theory to allow recursive methods (both filtering and smoothing) since the 1960s \citep{kalman60,kalman61}. 
\citet{jumarie_90} offers an interesting discussion on nonlinear, continuous-time filtering (using both a nonlinear dynamical plant and nonlinear observations). 
The tracking literature, in particular, has made heavy use of motion priors (in both continuous and discrete time) and has exploited the Markov property for efficient solutions \citep{maybeck79}.  
For the nonlinear, discrete-time case, \citet{bell_94} shows that Kalman filtering and smoothing with iterated relinearization is equivalent to Gauss-Newton on the full-state trajectory.  It is of no surprise that in vision and robotics,  discrete-time batch methods commonly exploit this sparsity property as well \citep{triggs00}.  In this paper, we make the (retrospectively obvious) observation that this sparsity can also be exploited in a batch, continuous-time context.
The result is that we derive a principled method to construct trajectory-smoothing terms for batch optimization (or factors in a factor-graph representation) based on a class of useful motion models; this paves the way to incorporate vehicle dynamics models, including exogenous inputs, to help with trajectory estimation.
\end{tolxemerg}

\begin{tolxemerg}{}{2pt}
Therefore, our main contribution is to emphasize the strong connection between classical estimation theory and machine learning via GP regression.  We use the fact that the inverse kernel matrix is sparse for a class of useful GP priors \citep{lindgren11,sarkka13} in a new way to efficiently implement nonlinear, GP regression for batch, continuous-time trajectory estimation.  We also show that this naturally leads to a subtle generalization of SLAM that we call {\em simultaneous trajectory estimation and mapping} (STEAM), with the difference being that chains of discrete {\em poses} are replaced with {\em Markovian trajectories} in order to incorporate continuous-time motion priors in an efficient way.  Finally, by using this GP paradigm, we are able to exploit the classic GP interpolation approach to query the trajectory at any time of interest in an efficient manner.
\end{tolxemerg}

\begin{tolxemerg}{300}{4pt}
This ability to query the trajectory at any time of interest in a principled way could be useful in a variety of situations.  For example, \citet{newman09} mapped a large urban area using a combination of stereo vision and laser rangefinders; the motion was estimated using the camera and the laser data were subsequently placed into a three-dimensional map based on this estimated motion.  Our method could provide a seamless means to (i) estimate the camera trajectory and then (ii) query this trajectory at every laser acquisition time.
\end{tolxemerg}

\begin{tolxemerg}{300}{4pt}
This paper is a significant extension of our recent conference paper \citep{barfoot_rss14}.
We build upon the exactly-sparse GP-regression approach that used \emph{linear}, time-varying SDEs, and show how to use GPs based on \emph{nonlinear}, time-varying SDEs, while maintaining the same level of sparsity as the linear case.
The algorithmic differences are discussed in detail and results are provided using comparable linear and nonlinear priors.
Furthermore, this paper shows how the block-tridiagonal sparsity of the kernel matrix can be exploited to improve the computational performance of hyperparameter training.
Finally, discussion is provided on using the GP interpolation equation for further state reduction at the cost of accuracy.
\end{tolxemerg}

The paper is organized as follows.  Section~\ref{sec:gpr} summarizes the general approach to batch state estimation via GP regression.  Section~\ref{sec:ltvsde} describes the particular class of GPs we use and elaborates on our main result concerning sparsity.  Section~\ref{sec:experiment} demonstrates this main result on a mobile robot example using a `constant-velocity' prior and compares the computational cost to methods that do not exploit the sparsity.  Section~\ref{sec:discussion} provides some discussion and Section~\ref{sec:conclusion} concludes the paper.

\section{Gaussian Process Regression}
\label{sec:gpr}

We take a {\em Gaussian-process-regression} approach to state estimation.  This allows us to (i) represent trajectories in continuous time (and therefore query the solution at any time of interest), and (ii) optimize our solution by iterating over the entire trajectory (recursive methods typically iterate at a single timestep).  

\begin{tolxemerg}{300}{4pt}
We will consider systems with a continuous-time, GP process model and a discrete-time, nonlinear measurement model:
\begin{eqnarray}
\mbf{x}(t) & \sim & \mathcal{GP}( \primean(t), \pricov(t,t^\prime) ), \qquad t_0 < t,t^\prime \\
\mbf{y}_\ik & = & \mbf{g}(\mbf{x}(t_\ik)) + \mbf{n}_\ik, \qquad  t_1 < \cdots < t_\iK,
\end{eqnarray}
where $\mbf{x}(t)$ is the state, $\primean(t)$ is the mean function, $\pricov(t,t^\prime)$ is the covariance function, $\mbf{y}_\ik$ are measurements, $\mbf{n}_\ik \sim \mathcal{N}(\mbf{0}, \mbf{R}_\ik )$ is Gaussian measurement noise, $\mbf{g}(\cdot)$ is a nonlinear measurement model, and $t_1 < \ldots < t_\iK$ is a sequence of measurement times.  For the moment, we do not consider the STEAM problem (i.e., the state does not include landmarks), but we will return to this case in our example later.
\end{tolxemerg}

We follow the approach of \citet{tong_ijrr13b} to set up our batch, GP state estimation problem.  We will first assume that we want to query the state at the measurement times, and will return to querying at other times later on.  
We start with an initial guess, $\opmean$, for the trajectory that will be improved iteratively.  At each iteration, we solve for the optimal perturbation, $\delta\mbf{x}^\star$, to our guess using GP regression, with our measurement model linearized about the current best guess.  

\begin{tolxemerg}{}{4pt}
The joint likelihood between the state and the measurements (both at the measurement times) is 
\begin{equation}
p\left( \bbm \mbf{x} \\ \mbf{y} \ebm \right) 
= \mathcal{N} \left( \bbm  \primean \\  \mbf{g} + \mbf{G} (\primean - \opmean)  \ebm, \bbm \pricov\; &  \pricov \mbf{G}^T \\  \mbf{G} \pricov\; &   \mbf{G} \pricov\mbf{G}^T  + \mbf{R}  \ebm  \right),
\end{equation}
where
\begin{small}
	\begin{gather*}
	\mbf{x} = \bbm \mbf{x}(t_0) \\ \vdots \\ \mbf{x}(t_\iK) \ebm, \quad \opmean = \bbm \opmean(t_0) \\ \vdots \\ \opmean(t_\iK) \ebm, \quad \primean = \bbm \primean(t_0)  \\ \vdots \\ \primean(t_\iK) \ebm, \\
	\mbf{y} = \bbm \mbf{y}_1 \\ \vdots \\ \mbf{y}_\iK \ebm,  \quad \mbf{g} = \bbm \mbf{g}(\opmean(t_1)) \\ \vdots \\ \mbf{g}(\opmean(t_\iK)) \ebm,  \quad \mbf{G} = \left. \frac{\partial \mbf{g}}{\partial \mbf{x}} \right|_{\opmean}, \\ \mbf{R} = \mbox{diag}\left( \mbf{R}_1, \ldots, \mbf{R}_\iK \right), \quad \pricov = \bbm \pricov(t_i, t_j) \ebm_{ij}.
	\end{gather*}%
\end{small}%
Note, the measurement model is linearized about our best guess so far.
We then have that the Gaussian posterior is
\begin{multline}
p(\mbf{x} | \mbf{y}) = \\ \quad \mathcal{N}\biggl(  
\underbrace{\primean + \pricov \mbf{G}^T \left( \mbf{G} \pricov \mbf{G}^T + \mbf{R} \right)^{-1} \bigl(\mbf{y} - \mbf{g} - \mbf{G} \left( \primean - \opmean \right) \bigr)}_{\text{$\postmean$, the posterior mean}},  \\
\underbrace{\pricov - \pricov \mbf{G}^T \left( \mbf{G} \pricov \mbf{G}^T + \mbf{R} \right)^{-1} \mbf{G} \pricov}_{\text{$\postcov$, the posterior covariance}}
\biggr).
\end{multline}
Letting $\mbfdel{x}^\star = \postmean - \opmean$, and rearranging the posterior mean expression using the Sherman-Morrison-Woodbury identity, we have 
\begin{equation}
\label{eq:gpgnlin}
\left( \pricov^{-1} + \mbf{G}^T \mbf{R}^{-1} \mbf{G} \right) \, \delta\mbf{x}^\star =  \pricov^{-1} ( \primean - \opmean) + \mbf{G}^T \mbf{R}^{-1}(\mbf{y} - \mbf{g}),
\end{equation}
which is a linear system for $\mbfdel{x}^\star$ and can be viewed as the solution to the associated {\em maximum a posteriori} (MAP) problem.
We know that the $\mbf{G}^T\mbf{R}^{-1}\mbf{G}$ term in~\eqref{eq:gpgnlin} is block-diagonal (assuming each measurement depends on the state at a single time), but in general $\pricov^{-1}$ could be dense, depending on the choice of GP prior.  At each iteration, we solve for $\mbfdel{x}^\star$ and then update the guess according to $\opmean \leftarrow \opmean + \mbfdel{x}^\star$; upon convergence we set $\postmean = \opmean$.
This is effectively Gauss-Newton optimization over the whole trajectory.
\end{tolxemerg}

\begin{tolxemerg}{}{4pt}
We may want to also query the state at some other time(s) of interest (in addition to the measurement times).  Though we could jointly estimate the trajectory at the measurement and query times, a better idea is to use GP interpolation after the solution at the measurement times converges \citep{rasmussen06,tong_ijrr13b} (see Section~\ref{sec:query} for more details).  GP interpolation automatically picks the correct interpolation scheme for a given prior; it arrives at the same answer as the joint approach (in the linear case), but at lower computational cost.
\end{tolxemerg}

In general, this GP approach has complexity $O(\iK^3 + \iK^2\iQry)$, where $\iK$ is the number of measurement times and $\iQry$ is the number of query times (the initial solve is $O(\iK^3)$ and the queries are $O(\iK^2\iQry)$).  This is quite expensive, and therefore we will seek to improve the cost by exploiting the structure of the matrices involved under a particular class of GP priors.

\section{A Class of Exactly Sparse GP Priors}
\label{sec:ltvsde}

\subsection{Linear, Time-Varying Stochastic Differential Equations}

\begin{tolxemerg}{400}{4pt}
We now show that the inverse kernel matrix is exactly sparse for a particular class of useful GP priors.  We consider GPs generated by linear, time-varying (LTV) stochastic differential equations (SDE) of the form
\begin{equation}
\label{eq:ltvsde}
\dot{\mbf{x}}(t) = \sdemat(t) \mbf{x}(t) + \mbf{v}(t) + \mbf{L}(t) \mbf{w}(t),
\end{equation}
where $\mbf{x}(t)$ is the state, $\mbf{v}(t)$ is a (known) exogenous input, $\mbf{w}(t)$ is white process noise, and $\sdemat(t)$, $\mbf{L}(t)$ are time-varying system matrices.  The process noise is given by
\begin{equation}
\mbf{w}(t) \sim \mathcal{GP}(\mbf{0}, \mbf{Q}_C \, \delta(t-t^\prime)),
\end{equation}
a (stationary) zero-mean {\em Gaussian process} (GP) with (symmetric, positive-definite) {\em power-spectral density matrix}, $\mbf{Q}_C$, and $\delta(\cdot)$ is the {\em Dirac delta function}.  
\end{tolxemerg}

The general solution to this LTV SDE \citep{maybeck79,stengel94} is
\begin{equation}
\label{eq:ltvsdesol}
\mbf{x}(t) = \mbs{\Phi}(t,t_0) \mbf{x}(t_0) + \int_{t_0}^t \mbs{\Phi}(t,s) \left( \mbf{v}(s) + \mbf{L}(s) \mbf{w}(s) \right) \, ds,
\end{equation}
where $\mbs{\Phi}(t,s)$ is known as the {\em transition matrix}.
From this model, we seek the mean and covariance functions for $\mbf{x}(t)$.

\subsubsection{Mean Function}

For the mean function, we take the expected value of~\eqref{eq:ltvsdesol}:
\begin{equation}
\label{eq:meanfunc}
\primean(t) = E[\mbf{x}(t)] = \mbs{\Phi}(t,t_0) \primean_0 + \int_{t_0}^t \mbs{\Phi}(t,s) \mbf{v}(s)  \, ds,
\end{equation}
where $\primean_0 = \primean(t_0)$ is the initial value of the mean.  
If we now have a sequence of measurement times, $t_0 < t_1 < t_2 < \cdots < t_\iK$, then we can write the mean at these times in {\em lifted form} as
\begin{equation}
\primean = \sdemat \mbf{v},
\end{equation}
where
\begin{gather}
\label{eq:L}
\primean = \bbm \primean(t_0) \\ \primean(t_1) \\ \vdots \\ \primean(t_\iK) \ebm, \; 
\mbf{v} = \bbm \primean_0 \\ \mbf{v}_1 \\ \vdots \\ \mbf{v}_\iK \ebm,  \;
\mbf{v}_\ik = \int_{t_{\ik-1}}^{t_\ik} \mbs{\Phi}(t_\ik,s) \mbf{v}(s) \, ds, \nonumber \\
\sdemat = \bbm \mbf{1} & \mbf{0} &  \cdots & \mbf{0}  &\mbf{0} \; \\
\mbs{\Phi}(t_1,t_0) & \mbf{1} &  \cdots & \mbf{0} & \mbf{0} \; \\
\mbs{\Phi}(t_2,t_0) & \mbs{\Phi}(t_2,t_1) & \ddots & \vdots & \vdots \; \\
\vdots & \vdots &  \ddots & \mbf{0} & \mbf{0} \; \\
\mbs{\Phi}(t_{\iK-1},t_0) & \mbs{\Phi}(t_{\iK-1},t_1) & \cdots & \mbf{1} & \mbf{0} \; \\
\mbs{\Phi}(t_\iK,t_0) & \mbs{\Phi}(t_\iK,t_1) &  \cdots & \mbs{\Phi}(t_\iK,t_{\iK-1}) & \mbf{1} \;
\ebm.
\end{gather}
Note that $\sdemat$, the {\em lifted transition matrix}, is lower-triangular. We arrive at this form by simply splitting up~\eqref{eq:meanfunc} into a sum of integrals between each pair of measurement times.

\subsubsection{Covariance Function}

For the covariance function, we take the second moment of~\eqref{eq:ltvsdesol} to arrive at
\begin{equation}
\label{eq:covfunc}
\begin{split}
\pricov(t,t^\prime) &= E\left[(\mbf{x}(t) -\primean(t))(\mbf{x}(t^\prime) -\primean(t^\prime))^T \right] \\
&=  \mbs{\Phi}(t,t_0) \pricov_0 \mbs{\Phi}(t^\prime,t_0)^T    \\
 & + \int_{t_0}^{\min(t,t^\prime)} \mbs{\Phi}(t,s) \mbf{L}(s)\mbf{Q}_C\mbf{L}(s)^T \mbs{\Phi}(t^\prime,s)^T  \, ds, 
\end{split}
\end{equation}
where $\pricov_0$ is the initial covariance at $t_0$ and we have assumed $E[\mbf{x}(t_0)\mbf{w}(t)^T] = \mbf{0}$.  
Using a sequence of measurement times, $t_0 < t_1 < t_2 < \cdots < t_\iK$, we can write the covariance between two times as%
\begin{small}
	\begin{equation}
	\label{eq:K}
	\pricov(t_i, t_j) \hspace{-0.05cm} = \hspace{-0.05cm} \left\{\!\!\!\!\begin{array}{cl}
	\mbs{\Phi}(t_i,t_j) \left( \sum_{r=0}^j \mbs{\Phi}(t_j,t_r) \mbf{Q}_r \mbs{\Phi}(t_j,t_r)^T \right)   & t_j < t_i \\ 
	\sum_{r=0}^i \mbs{\Phi}(t_i,t_r) \mbf{Q}_r \mbs{\Phi}(t_i,t_r)^T    & t_i = t_j \\
	\left(\sum_{r=0}^i \mbs{\Phi}(t_i,t_r) \mbf{Q}_r \mbs{\Phi}(t_i,t_r)^T  \right) \mbs{\Phi}(t_j,t_i)^T   & t_i < t_j 
	\end{array}\right. 
	\end{equation}
\end{small}%
where
\begin{equation}
\label{eq:Qi}
\mbf{Q}_{\ik} = \int_{t_{\ik-1}}^{t_\ik} \mbs{\Phi}(t_\ik,s) \mbf{L}(s) \mbf{Q}_C \mbf{L}(s)^T \mbs{\Phi}(t_\ik,s)^T \, ds,
\end{equation}
for $\ik=1\ldots \iK$ and $\mbf{Q}_0 = \pricov_0$ (to keep the notation simple).  Given this preparation, we are now ready to state the main sparsity result that we will exploit in the rest of the paper.

\begin{lemma}
	\label{lem1}
	Let $t_0< t_1 < t_2 <  \cdots < t_\iK$ be a monotonically increasing sequence of time values.  Using~\eqref{eq:K}, we define the $(\iK+1) \times (\iK+1)$ kernel matrix (i.e., the prior covariance matrix between all pairs of times), $\pricov = \left[ \pricov(t_i,t_j) \right]_{ij}$.
	Then, we can factor $\pricov$ according to a lower-diagonal-upper decomposition,
	\begin{equation}
	\pricov = \sdemat \mbf{Q} \sdemat^T,
	\end{equation}
	where $\sdemat$ is the lower-triangular matrix given in~\eqref{eq:L} and $\mbf{Q} = \diag\left( \pricov_0, \mbf{Q}_1, \ldots, \mbf{Q}_\iK \right)$ with $\mbf{Q}_\ik$ given in~\eqref{eq:Qi}.
\end{lemma}

\begin{proof}
	Straightforward to verify by substitution. \qed
\end{proof}

\begin{theorem}
	\label{thm1}
	The inverse of the kernel matrix constructed in Lemma~\ref{lem1}, $\pricov^{-1}$, is exactly sparse (block-tridiagonal). 
\end{theorem}

\begin{proof}
	The decomposition of $\pricov$ in Lemma~\ref{lem1} provides
	\begin{equation}
	\pricov^{-1} = (\sdemat\mbf{Q}\sdemat^T)^{-1} = \sdemat^{-T} \mbf{Q}^{-1} \sdemat^{-1}.
	\end{equation}
	where the inverse of the lifted transition matrix is%
	%\begin{small}
		\begin{equation}
		\sdemat^{-1} = \bbm \mbf{1} & \mbf{0} &  \cdots & \mbf{0} & \mbf{0} \; \\
		-\mbs{\Phi}(t_1,t_0) & \mbf{1} &  \cdots & \mbf{0} & \mbf{0} \; \\
		\mbf{0} & -\mbs{\Phi}(t_2,t_1) &  \ddots & \vdots & \vdots \; \\
		\mbf{0} & \mbf{0} &  \ddots & \mbf{0} & \mbf{0} \; \\
		\vdots & \vdots &  \cdots & \mbf{1} & \mbf{0} \; \\
		\mbf{0} & \mbf{0} &  \cdots & -\mbs{\Phi}(t_\iK,t_{\iK-1}) & \mbf{1}\;
		\ebm,
		\end{equation}
	%\end{small}%
	and $\mbf{Q}^{-1}$ is block-diagonal.  The block-tridiagonal property of $\pricov^{-1}$ follows by substitution and multiplication. \qed
\end{proof}

While the block-tridiagonal property stated in Theorem 1 has been exploited in vision and robotics for a long time \citep{lu97,triggs00,thrun06}, the usual route to this point is to begin by converting the continuous-time motion model to discrete time and then to directly formulate a maximum a posteriori optimization problem; this bypasses writing out the full expression for $\pricov$ and jumps to an expression for $\pricov^{-1}$.  However, we require expressions for both $\pricov$ and $\pricov^{-1}$ to carry out our GP reinterpretation and facilitate querying the trajectory at an arbitrary time (through interpolation).  That said, it is also worth noting we have not needed to convert the motion model to discrete time and have made no approximations thus far.

\begin{tolxemerg}{300}{4pt}
Given the above results, the prior over the state (at the measurement times) can be written as
\begin{equation}
\mbf{x} \sim \mathcal{N}\left( \primean, \pricov \right) = \mathcal{N} \left( \sdemat \mbf{v}, \sdemat\mbf{Q}\sdemat^T \right).
\end{equation}
More importantly, using the result of Theorem~\ref{thm1} in~\eqref{eq:gpgnlin} gives
\begin{multline}
\label{eq:gpgnlin2}
\bigl(\overbrace{\sdemat^{-T}\mbf{Q}^{-1}\sdemat^{-1} + \mbf{G}^T \mbf{R}^{-1} \mbf{G}}^{\mbox{\small block-tridiagonal}}\bigr) \, \delta\mbf{x}^\star \\
= \sdemat^{-T}\mbf{Q}^{-1}( \mbf{v} - \sdemat^{-1} \opmean) + \mbf{G}^T \mbf{R}^{-1}(\mbf{y} - \mbf{g}).
\end{multline}
which can be solved in $O(\iK)$ time (at each iteration), using a sparse solver (e.g., sparse Cholesky decomposition then forward-backward passes). In fact, in the case of a linear measurement model, one such solver is the classical, forward-backward smoother (i.e., Kalman or Rauch--Tung--Striebel smoother) \citep{bell_94}. Put another way, the forward-backward smoother is possible {\em because} of the sparse structure of~\eqref{eq:gpgnlin2}.  For nonlinear measurement models, our scheme iterates over the whole trajectory; it is therefore related to, but not the same as, the `extended' version of the forward-backward smoother \citep{sarkka13b,sarkka13c,sarkka13}.
\end{tolxemerg}

\begin{figure}
	\includegraphics[width=\columnwidth]{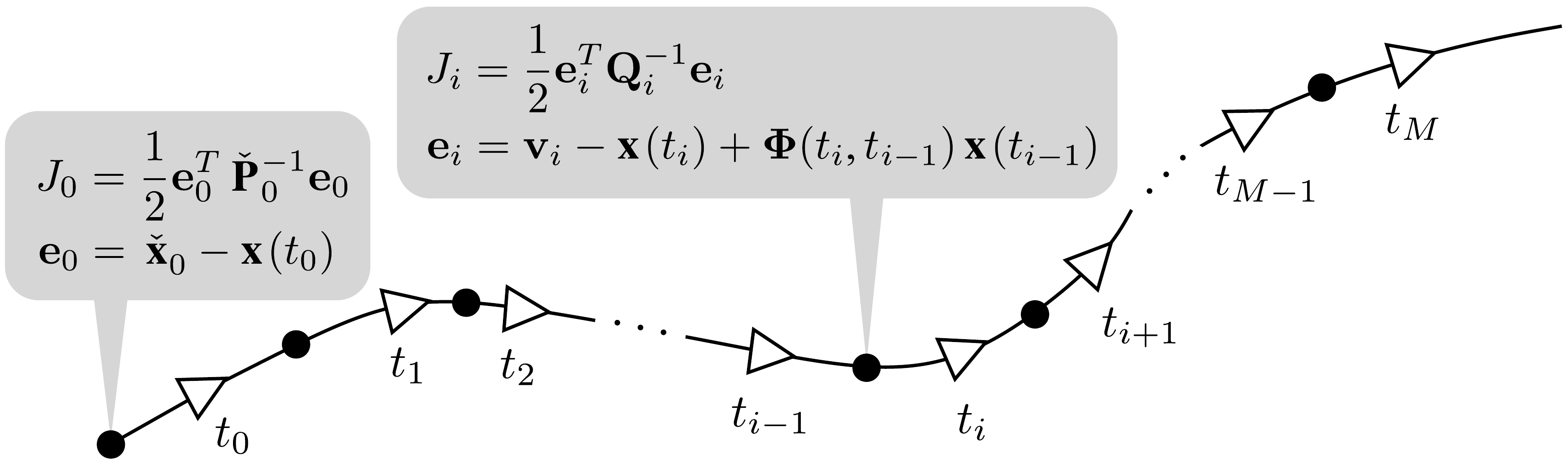}
	\caption{Although we began with a continuous-time prior to smooth our trajectory, the class of exactly sparse GPs results in only $\iK+1$ smoothing terms, $J_i$, in the associated optimization problem, the solution of which is~\eqref{eq:gpgnlin2}.  We can depict these graphically as factors (black dots) in a factor-graph representation of the prior \citep{dellaert06}.  The triangles are {\em trajectory states}, the nature of which depends on the choice of prior.}
	\label{fig:factors}
\end{figure}

\begin{tolxemerg}{300}{4pt}
Perhaps the most interesting outcome of Theorem~\ref{thm1} is that, although we are using a continuous-time prior to smooth our trajectory, at implementation we require only $\iK+1$ smoothing terms in the associated MAP optimization problem:  $\iK$ between consecutive pairs of measurement times plus $1$ at the initial time (unless we are also estimating a map).  As mentioned before, this is the same form that we would have arrived at had we started by converting our motion model to discrete time at the beginning \citep{lu97,triggs00,thrun06}.  This equivalence has been noticed before for recursive solutions to estimation problems with a continuous-time state and discrete-time measurements \citep{sarkka06}, but not in the batch scenario.  Figure~\ref{fig:factors} depicts the $\iK+1$ smoothing terms in a factor-graph representation of the prior \citep{dellaert06,kaess12}.   
\end{tolxemerg}

However, while the form of the smoothing terms/factors is similar to the original discrete-time form introduced by \citet{lu97}, our approach provides a principled method for their construction, starting from the continuous-time motion model.  Critically, we stress that the state being estimated must be {\em Markovian} in order to obtain the desirable sparse structure.  In the experiment section, we will investigate a common GP prior, namely the `constant-velocity' or white-noise-on-acceleration model: $\ddot{\mbf{p}}(t) = \mbf{w}(t)$, 
where $\mbf{p}(t)$ represents position.  For this choice of model, $\mbf{p}(t)$ is not Markovian, but
\begin{equation}
\mbf{x}(t) = \bbm \mbf{p}(t) \\ \dot{\mbf{p}}(t) \ebm,
\end{equation}
is.  This implies that, if we want to use the `constant-velocity' prior and enjoy the sparse structure without approximation, we must estimate a stacked state with both position and velocity.  Marginalizing out the velocity variables fills in the inverse kernel matrix, thereby destroying the sparsity.

If all we cared about was estimating the value of the state at the measurement times, our GP paradigm arguably offers little beyond a reinterpretation of the usual discrete-time approach to batch estimation.  However, by taking the time to set up the problem in this manner, we can now query the trajectory at {\em any} time of interest using the classic interpolation scheme that is inherent to GP regression \citep{tong_ijrr13b}.

\subsubsection{Querying the Trajectory}
\label{sec:query}

As discussed in Section~\ref{sec:gpr}, after we solve for the trajectory at the measurement times, we may want to query it at other times of interest.  
This operation also benefits greatly from the sparse structure.  
To keep things simple, we consider a single query time, $t_\ik \leq \tau < t_{\ik+1}$ (see Figure~\ref{fig:setup}).  
The standard linear GP interpolation formulas \citep{rasmussen06,tong_ijrr13b} are 
\begin{subequations}
	\label{eq:lininterp}
	\begin{align}
	\postmean(\tau) &= \primean(\tau) + \pricov(\tau) \pricov^{-1} (\postmean - \primean), \\
	%\opmean(\tau) = \primean(\tau) + \pricov(\tau) \pricov^{-1} (\opmean - \primean),
	\label{eq:lininterp_covariance}
	\postcov(\tau,\tau) &= \pricov(\tau,\tau) + \pricov(\tau) \pricov^{-1} \Bigl( \postcov - \pricov \Bigr)\pricov^{-T}\pricov(\tau)^T ,
	\end{align}
\end{subequations}%
where $\pricov(\tau) = \bbm \pricov(\tau,t_0) & \cdots & \pricov(\tau,t_\iK) \ebm$. 
Note, we write \eqref{eq:lininterp_covariance} in a less common, but equivalent, form, as we intend to exploit the sparsity of the product $\pricov(\tau) \pricov^{-1}$.

For the mean function at the query time, we simply have
\begin{equation}
\label{eq:querymeanfunc}
\primean(\tau) = \mbs{\Phi}(\tau,t_\ik) \primean_\ik + \int_{t_\ik}^\tau \mbs{\Phi}(\tau,s) \mbf{v}(s) \, ds,
\end{equation}
which can be evaluated in $O(1)$ time.
For the covariance function at the query time, we have 
\begin{multline}
\label{eq:querycovfunc}
\pricov(\tau,\tau) = \mbs{\Phi}(\tau,t_\ik) \pricov(t_\ik,t_\ik) \mbs{\Phi}(\tau,t_\ik)^T + \\
\int_{t_\ik}^\tau \mbs{\Phi}(\tau,s) \mbf{L}(s) \mbf{Q}_C \mbf{L}(s)^T \mbs{\Phi}(\tau,s)^T \, ds,
\end{multline}
which is also $O(1)$ to evaluate.

The computational savings come from the sparsity of the product $\pricov(\tau) \pricov^{-1}$, which represents the burden of the cost in the interpolation formula.  After some effort, it turns out we can write $\pricov(\tau)$ as
\begin{equation}
\pricov(\tau) = \mbf{V}(\tau)\sdemat^T,
\end{equation}
where $\sdemat$ was defined before,
\begin{align}
\mbf{V}(\tau) &= \Bigl[ \begin{matrix} \mbs{\Phi}(\tau,t_\ik) \mbs{\Phi}(t_\ik,t_0) \pricov_0 &  \mbs{\Phi}(\tau,t_\ik) \mbs{\Phi}(t_\ik,t_1) \mbf{Q}_1 & \cdots \end{matrix} \nonumber \\ 
& \quad\; \begin{matrix} \cdots &  \mbs{\Phi}(\tau,t_\ik) \mbs{\Phi}(t_\ik,t_{\ik-1})  \mbf{Q}_{\ik-1} & \mbs{\Phi}(\tau,t_\ik) \mbf{Q}_\ik  & \cdots \end{matrix} \nonumber \\ 
& \quad\quad\quad\quad\quad\quad\quad \begin{matrix} \cdots &  \mbf{Q}_{\tau} \mbs{\Phi}(t_{\ik+1},\tau)^T & \mbf{0} & \cdots & \mbf{0}  \end{matrix} \Bigr],
\end{align}
and
\begin{equation}
\mbf{Q}_{\tau} = \int_{t_\ik}^{\tau} \mbs{\Phi}(\tau,s) \mbf{L}(s) \mbf{Q}_C \mbf{L}(s)^T \mbs{\Phi}(\tau,s)^T \, ds.
\end{equation}
Returning to the desired product, we have
\begin{equation}
\pricov(\tau) \pricov^{-1} = \mbf{V}(\tau) \underbrace{\sdemat^T \sdemat^{-T}}_{\mbf{1}} \mbf{Q}^{-1} \sdemat^{-1} = \mbf{V}(\tau) \mbf{Q}^{-1} \sdemat^{-1}.
\end{equation}
Since $\mbf{Q}^{-1}$ is block-diagonal, and $\sdemat^{-1}$ has only the main diagonal and the one below it non-zero, we can evaluate the product very efficiently. Note, there are exactly two non-zero block-columns:
\begin{equation}
\label{eq:KtauKinv}
\pricov(\tau) \pricov^{-1} = \Bigl[ \, \mbf{0} \;\; \cdots \;\;  \mbf{0} \;\; \underbrace{\mbs{\Lambda}(\tau)}_{\substack{\text{\small block}\\\text{\small col. $\ik$}}} \;\; \underbrace{\mbs{\Psi}(\tau)}_{\substack{\text{\small block}\\\text{\small col. \!$\ik\!+\!1$}}} \; \mbf{0} \;\; \cdots \;\; \mbf{0}  \, \Bigr],
\end{equation}
where
\begin{subequations}
	\begin{align}
	\mbs{\Lambda}(\tau) &= \mbs{\Phi}(\tau,t_\ik) - \mbs{\Psi}(\tau) \mbs{\Phi}(t_{\ik+1},t_\ik), \\
	\mbs{\Psi}(\tau) &= \mbf{Q}_{\tau} \mbs{\Phi}(t_{\ik+1},\tau)^T \mbf{Q}_{\ik+1}^{-1}.
	\end{align}
\end{subequations}
Inserting this into~\eqref{eq:lininterp}, we have
\begin{subequations}
	\begin{align}
	\postmean(\tau) &= \primean(\tau) + 
	\bbm \mbs{\Lambda}(\tau) & \mbs{\Psi}(\tau) \ebm
	\bbm \postmean_\ik - \primean_\ik \\ \postmean_{\ik+1} - \primean_{\ik+1} \ebm, \\
	\begin{split}
		\postcov(\tau,\tau) &= \pricov(\tau,\tau) 
		+ \bbm \mbs{\Lambda}(\tau) & \mbs{\Psi}(\tau) \ebm
		\biggl(
		\bbm \postcov_{\ik, \ik} & \postcov_{\ik, \ik+1} \\ \postcov_{\ik+1,\ik} & \postcov_{\ik+1,\ik+1} \ebm \\
		& \quad\quad\quad - 
		\bbm \pricov_{\ik, \ik} & \pricov_{\ik, \ik+1} \\ \pricov_{\ik+1,\ik} & \pricov_{\ik+1,\ik+1} \ebm
		\biggr)
		\bbm \mbs{\Lambda}(\tau)^T \\ \mbs{\Psi}(\tau)^T \ebm,
	\end{split}
	\end{align}
\end{subequations}
which is a linear combination of just the terms from $t_\ik$ and $t_{\ik+1}$.  If the query time is beyond the last measurement time, $t_\iK < \tau$, the expression will involve only the term at $t_\iK$ and represents extrapolation/prediction rather than interpolation/smoothing.  In summary, to query the trajectory at a single time of interest is $O(1)$ complexity.

\subsubsection{Interpolating Measurement Times}
\label{sec:measinterp}

Thus far, we have shown that by storing the \emph{Markovian}, $\mbf{x}(t_\ik)$, state at every measurement time, $t_\ik, \ik = 1 \hdots \iK$, we are able to perform Gaussian-process regression in $O(\iK)$ time and then query for other times of interest $\mbf{x}(\tau)$, all without approximation.
Given the ability to interpolate, storing the state at every measurement time may be excessive, especially in a scenario where the measurement rate is high in comparison to the smoothness of the robot kinematics (e.g. a 1000\;Hz IMU or individually timestamped lidar measurements, mounted on a slow indoor platform). 

\citet{tong_ijrr13b} discuss a scheme to remove some of the measurement times from the initial solve, which further reduces computational cost with some loss of accuracy. 
By estimating $\mbf{x}$ at only some \emph{keytimes}, $t_\ikf, \ikf = 0 \hdots \iKf$, which may or may not align with some subset of the measurement times, the interpolation equation can be used to modify our measurement model as follows,
\begin{align}
\mbf{y}_\ik &= \mbf{g}\left(\opmean(t_\ik)\right) \nonumber \\
&= \mbf{g}\hspace{-0.05cm}\left(\hspace{-0.02cm} \primean(t_\ik) + 
\bbm \mbs{\Lambda}(t_\ik) & \mbs{\Psi}(t_\ik) \ebm
\bbm \mbf{x}_{\text{op},\ikf} - \primean_\ikf \\ \mbf{x}_{\text{op},\ikf+1} - \primean_{\ikf+1} \ebm \right)\hspace{-0.08cm},
\end{align}
where $t_\ikf \leq t_\ik < t_{\ikf+1}$. The matrices $\mbs{\Lambda}(t_\ik)$ and $\mbs{\Psi}(t_\ik)$ are constructed according to \eqref{eq:KtauKinv}, where the measurement time, $t_\ik$, is now the query time and the bounding times (previously $t_\ik $ and $ t_{\ik+1}$) become the keytimes $t_\ikf $ and $ t_{\ikf+1}$.
The effect on $\mbf{G}$ is that each block row now has two adjacent non-zero block columns, rather than one.
This causes the structure of $\mbf{G}^T \mbf{R}^{-1} \mbf{G}$ to change from being block-diagonal to block-tridiagonal;
the structure of $\mbf{G}^T \mbf{R}^{-1} \mbf{G}$ is of particular importance because it directly affects the complexity of solving \eqref{eq:gpgnlin2}.
Fortunately, the complexity is unaffected because our prior term is also block-tridiagonal.

\begin{tolxemerg}{}{4pt}
An important intuition is that the interpolated state at some measurement time, $t_\ik$, depends only on the state estimates, $\mbf{x}_{\text{op},\ikf}$ and $\mbf{x}_{\text{op},\ikf+1}$, and the prior terms, $\primean(t)$ and $\pricov(t,t')$.
Therefore, the effect of estimating $\mbf{x}$ at some reduced number of times is that we obtain a smoothed solution;
any details provided by high frequency measurements between two times of interest, $t_\ikf$ and $t_{\ikf+1}$, are lost.
Although this detail information is smoothed over, there are obvious computational savings in having a smaller state and a subtle benefit regarding the prevention of overfitting measurements.
With respect to fitting, the glaring issue with this scheme is that there is not a principled method to determine an appropriate spacing for the \emph{keytimes}, $t_\ikf$, such that we could guarantee a bound on the loss of accuracy.
Learning from the parametric continuous-time estimation schemes, which suffer from a similar issue, it is reasonable to start with some uniform spacing and add additional \emph{keytimes} based on the results of a normalized-innovation-squared test between each pair of \emph{keytimes} \citep{oth_cvpr13}.
Some experimentation is necessary to better understand which approach to state discretization is best in which situation.
\end{tolxemerg}

\subsection{Nonlinear, Time-Varying Stochastic Differential Equations}

\begin{tolxemerg}{}{4pt}
In reality, most systems are inherently nonlinear and cannot be accurately described by a LTV SDE in the form of \eqref{eq:ltvsde}.
Moving forward, we show how our results concerning sparsity can be applied to nonlinear, time-varying (NTV) stochastic differential equations (SDE) of the form
\begin{equation}
\label{eq:ntvsde}
\dot{\mbf{x}}(t) = \mbf{f}(\mbf{x}(t),\mbf{u}(t),\mbf{w}(t)),
\end{equation}
where $\mbf{f}(\cdot)$ is a nonlinear function, $\mbf{x}(t)$ is the state, $\mbf{u}(t)$ is a (known) exogenous input, and $\mbf{w}(t)$ is white process noise.
To perform GP regression with a nonlinear process model, we begin by linearizing the SDE about a continuous-time operating point $\opmean(t)$,
\begin{align}
\dot{\mbf{x}}(t) &= \mbf{f}(\mbf{x}(t),\mbf{u}(t),\mbf{w}(t)) \nonumber \\
\label{eq:ntvsde_FLlin}
&\approx \mbf{f}(\opmean(t),\mbf{u}(t),\mbf{0}) +
\mbf{F}(t) (\mbf{x}(t) - \opmean(t)) +
\mbf{L}(t) \mbf{w}(t).
\end{align}
where
\begin{align}
\mbf{F}(t) = \left. \frac{\partial \mbf{f}}{\partial \mbf{x}} \right|_{\opmean(t),\mbf{u}(t),\mbf{0}}, \quad
\mbf{L}(t) = \left. \frac{\partial \mbf{f}}{\partial \mbf{w}} \right|_{\opmean(t),\mbf{u}(t),\mbf{0}}.
\end{align}
Setting
\begin{equation}
\label{eq:ntvsde_vlin}
\mbf{v}(t) = \mbf{f}(\opmean(t),\mbf{u}(t),\mbf{0}) - \mbf{F}(t) \opmean(t)
\end{equation}
lets us rewrite \eqref{eq:ntvsde_FLlin} in the familiar LTV SDE form,
\begin{align}
\label{eq:ntvsde_lin}
\dot{\mbf{x}}(t) &\approx \mbf{F}(t) \mbf{x}(t) + \mbf{v}(t) + \mbf{L}(t) \mbf{w}(t),
\end{align}
where $\mbf{F}(t)$, $\mbf{v}(t)$, and $\mbf{L}(t)$ are known functions of time, since $\opmean(t)$ is known.
Setting the operating point, $\opmean(t)$, to our best guess of the underlying trajectory at each iteration of GP regression, we note a similarity in nature to the discrete-time, recursive method of \citet{bell_94}; in essence, our approach offers a continuous-discrete version of the Gauss-Newton estimator, using the Gaussian-process-regressor type of approximate bridging between the measurement times.
\end{tolxemerg}

\subsubsection{Mean and Covariance Function}
\label{sec:nonlinear_mean_and_cov}

Although the equations for calculating the mean, $\primean$, and covariance, $\pricov$, remain the same as in \eqref{eq:meanfunc} and \eqref{eq:covfunc}, there are a few algorithmic differences and issues that arise due to the new dependence on the continuous-time operating point, $\opmean(t)$.
An obvious difference in contrast to the GP regression using a LTV SDE is that we now must recalculate $\primean$ and $\pricov$ at each iteration of the optimization (since the linearization point is updated).
The main algorithmic issue that presents itself is that the calculation of $\primean$ and $\pricov$ require integrations involving $\mbf{F}(t)$, $\mbf{v}(t)$, and $\mbf{L}(t)$, which in turn require the evaluation of $\opmean(t)$ over the time period $t \in [t_0, t_\iK]$.
Since an estimate of the posterior mean, $\opmean$, is only stored at times of interest, we must make use of the efficient (for our particular choice of process model) GP interpolation equation derived in Section~\ref{sec:query},
\begin{equation}
\label{eq:nonlininterp}
\opmean(\tau) = \primean(\tau) + \pricov(\tau) \pricov^{-1} (\opmean - \primean).
\end{equation}
The problem is that the above interpolation depends again on $\primean$ and $\pricov$, which are the prior variables for which we want to solve.
To rectify this circular dependence, we take advantage of the iterative nature of GP regression and choose to evaluate \eqref{eq:nonlininterp} using the values of $\primean(\tau)$, $\primean$, and $\pricov(\tau)\pricov^{-1}$ from the previous iteration.

\begin{tolxemerg}{}{4pt}
Another issue with nonlinear process models is that identifying an analytical expression for the state transition matrix, $\mbs{\Phi}(t,s)$ (which is dependent on the form of $\mbf{F}(t)$), can be very challenging.
Fortunately, the transition matrix can also be calculated numerically via the integration of the normalized fundamental matrix, $\mbs{\Upsilon}(t)$, where
\begin{equation}
\label{eq:normfundmat}
\mbsdot{\Upsilon}(t) = \mbf{F}(t)\mbs{\Upsilon}(t), \quad \mbs{\Upsilon}(0) = \mbf{1}.
\end{equation}
Storing $\mbs{\Upsilon}(t)$ at times of interest, the transition matrix can then be computed using
\begin{equation}
\mbs{\Phi}(t,s) = \mbs{\Upsilon}(t)\mbs{\Upsilon}(s)^{-1}.
\end{equation}
In contrast to the LTV SDE system, using a nonlinear process model causes additional computational costs (primarily due to numerical integrations), but complexity remains linear in the length of the trajectory (at each iteration), and therefore continues to be computationally tractable.
\end{tolxemerg}

\subsubsection{Querying the Trajectory}
In this section, we discuss the algorithmic details of the GP interpolation procedure as it pertains to a NTV SDE process model.
Recall the standard linear GP interpolation formulas presented in \eqref{eq:lininterp}, for a single query time, $t_\ik \leq \tau < t_{\ik+1}$:
\begin{subequations}
	\begin{align*}
	\postmean(\tau) &= \primean(\tau) + \pricov(\tau) \pricov^{-1} (\postmean - \primean), \\
	\postcov(\tau,\tau) &= \pricov(\tau,\tau) + \pricov(\tau) \pricov^{-1} \Bigl( \postcov - \pricov \Bigr)\pricov^{-T}\pricov(\tau)^T .
	\end{align*}
\end{subequations}
The final iteration of GP regression (where $\postmean = \opmean$), provides values for $\primean$, $\pricov$, $\postmean$, and $\postcov$; however, obtaining values for $\primean(\tau)$, $\pricov(\tau) \pricov^{-1}$ (recall sparse structure in \eqref{eq:KtauKinv}), and $\pricov(\tau, \tau)$ is not trivial.
The suggestion made in Section~\ref{sec:nonlinear_mean_and_cov} was to use values from the previous (or in this case, final) iteration.
Thus far, the method of storing these continuous-time functions has been left ambiguous.

The naive way to store $\primean(\tau)$, $\pricov(\tau) \pricov^{-1}$, and $\pricov(\tau, \tau)$ is to keep the values at all numerical integration timesteps;
the memory requirement of this is proportional to the length of the trajectory.
During the optimization procedure this naive method may in fact be preferable as it reduces computation time in lieu of additional storage (which is fairly cheap using current technology). 
For long-term storage, a method that uses a smaller memory footprint (at the cost of additional computation) may be desirable.
The remainder of this section will focus on identifying the minimal storage requirements, such that queries remain $O(1)$ complexity.

\begin{tolxemerg}{}{4pt}
We begin by examining the mean function; substituting \eqref{eq:querymeanfunc} and \eqref{eq:ntvsde_vlin}, into the GP interpolation formula above:
\begin{multline}
\label{eq:numint_querymean}
\postmean(\tau) = \mbs{\Phi}(\tau,t_\ik) \primean_\ik + \int_{t_\ik}^\tau \mbs{\Phi}(\tau,s) \bigl( \mbf{f}(\postmean(s),\mbf{u}(s),\mbf{0}) \\
 - \mbf{F}(s) \postmean(s) \bigr) \, ds + \pricov(\tau) \pricov^{-1} (\postmean - \primean).
\end{multline}
It is straightforward to see how $\postmean(\tau)$ can be simultaneously numerically integrated with the normalized fundamental matrix from \eqref{eq:normfundmat}, as long as we can evaluate the term $\pricov(\tau) \pricov^{-1}$.
Although we chose to examine the mean function, a similar conclusion can be drawn by examining the covariance function in \eqref{eq:querycovfunc}.
Recalling the sparse structure of $\pricov(\tau) \pricov^{-1}$ (see \eqref{eq:KtauKinv}) for an LTV SDE process model, where
\begin{align*}
\mbs{\Lambda}(\tau) &= \mbs{\Phi}(\tau,t_\ik) - \mbs{\Psi}(\tau) \mbs{\Phi}(t_{\ik+1},t_\ik), \\
\mbs{\Psi}(\tau) &= \mbf{Q}_{\tau} \mbs{\Phi}(t_{\ik+1},\tau)^T \mbf{Q}_{\ik+1}^{-1},
\end{align*}
we are able to draw two conclusions.
First, in the case that an analytical expression for $\mbs{\Phi}(t,s)$ is unavailable, we must store $\mbs{\Upsilon}(t_\ik)$ at the times of interest $t_\ik, \ik = 1 \hdots \iK$, since any numerical integration will involve `future' values of $\mbs{\Upsilon}(t)$ (via the evaluation of $\mbs{\Phi}(t_{\ik+1},t_\ik)$, $\mbs{\Phi}(t_{\ik+1},\tau)$, and $\mbf{Q}_{\ik+1}^{-1}$).
Second, in the case that $\mbf{L}(t)$ is a time-varying matrix, we must store $\mbf{Q}_\ik^{-1}, \ik = 1 \hdots \iK$, since the evaluation of $\mbf{Q}_{\ik}^{-1}$, requires $\mbf{L}(s)$ (evaluated at $\postmean(s)$) over the time period $[t_{\ik-1}, t_{\ik}]$.
The memory requirements of this alternative are still proportional to the length of the trajectory, but are greatly reduced in contrast to the naive method;
the added cost is that any new query requires numerical integration from the nearest time $t_\ik$ to the query time $\tau$ (which is of order $O(1)$).
\end{tolxemerg}

\subsection{Training the Hyperparameters}
\label{sec:hyper_training}

\begin{tolxemerg}{}{2pt}
As with any GP regression, we have {\em hyperparameters} associated with our covariance function, namely $\mbf{Q}_C$, which affect the smoothness and length scale of the class of functions we are considering as motion priors. The covariances of the measurement noises can also be unknown or uncertain. The standard approach to selecting these parameters is to use a training dataset (with ground-truth), and perform optimization using the log marginal likelihood (log-evidence) or its approximation as the objective function \citep{rasmussen06}.
We begin with the marginal likelihood equation, 
\begin{align}
\begin{split}
\log p(\mbf{y}|\mbf{Q}_C) &= 
-\frac{1}{2} (\mbf{y} - \primean)^T \mbf{P}_\ihy^{-1} (\mbf{y} - \primean) \\
& \quad\quad - \frac{1}{2} \log |\mbf{P}_\ihy| 
- \frac{n}{2} \log 2\pi,
\end{split} \\
\mbf{P}_\ihy &= \pricov(\mbf{Q}_C) + \sigma_\ihy^2 \mbf{1},
\end{align}
where $\mbf{y}$ is a stacked vector of state observations (ground-truth measurements) with additive noise $\mathcal{N}(\mbf{0},\sigma_\ihy^2\mbf{1})$, $\primean$ is a stacked vector of the mean equation evaluated at the observation times, $t_\ihy, \ihy = 1 \hdots \iHy$, and $\pricov$ is the covariance matrix associated with $\primean$ and generated using the hyperparameters $\mbf{Q}_C$.
Taking partial derivatives of the marginal likelihood with respect to the hyperparameters, we get
\begin{multline}
\label{eq:partial_marglike}
\frac{\partial}{\partial \mbf{Q}_{C_{ij}}} \log p(\mbf{y}|\mbf{Q}_C) = 
\frac{1}{2} (\mbf{y} - \primean)^T \mbf{P}_\ihy^{-1} \frac{\partial \mbf{P}_\ihy}{\partial \mbf{Q}_{C_{ij}}} \mbf{P}_\ihy^{-1} (\mbf{y} - \primean) 
\\ - \frac{1}{2} \tr \left( \mbf{P}_\ihy^{-1}\frac{\partial \mbf{P}_\ihy}{\partial \mbf{Q}_{C_{ij}}} \right),
\end{multline}%
where we have used that $\frac{\partial \mbf{P}_\ihy^{-1}}{\partial \mbf{Q}_{C_{ij}}} = -\mbf{P}_\ihy^{-1} \frac{\partial \mbf{P}_\ihy}{\partial \mbf{Q}_{C_{ij}}} \mbf{P}_\ihy^{-1}$.
\end{tolxemerg}

\begin{tolxemerg}{}{2pt}
The typical complexity of hyperparameter training is bottlenecked at $O(\iHy^3)$ due to the inversion of $\mbf{P}_\ihy$ (which is typically dense).
Given $\mbf{P}_\ihy^{-1}$, the complexity is then typically bottlenecked at $O(\iHy^2)$ due to the calculation of $\frac{\partial \mbf{P}_\ihy}{\partial \mbf{Q}_{C_{ij}}}$.
Fortunately, in the present case, the computation of the log marginal likelihood can also be done efficiently due to the sparseness of the inverse kernel matrix, $\pricov^{-1}$.
Despite the addition of observation noise, $\sigma_\ihy^2\mbf{1}$, causing $\mbf{P}_\ihy^{-1}$ to be a dense matrix, we are still able to take advantage of our sparsity using the Sherman-Morrison-Woodbury identity:
\begin{align}
\mbf{P}_\ihy^{-1} 
= \pricov^{-1} - \pricov^{-1} \left(\pricov^{-1} + \frac{1}{\sigma_\ihy^2}\mbf{1} \right)^{-1} \pricov^{-1} 
\hspace{-0.08cm} = \mbf{F}^{-T} \mbf{Q}_\ihy^{-1} \mbf{F}^{-1},
\end{align}
where we define
\begin{equation}
\mbf{Q}_\ihy^{-1} = \left( \mbf{Q}^{-1} - \mbf{Q}^{-1} \mbf{F}^{-1} \left(\pricov^{-1} + \frac{1}{\sigma_\ihy^2}\mbf{1} \right)^{-1} \mbf{F}^{-T} \mbf{Q}^{-1} \right).
\end{equation}
Although computing $\mbf{P}_\ihy^{-1}$ explicitly is of order $O(\iHy^2)$, we note that the product with a vector, $\mbf{P}_\ihy^{-1}\mbf{v}$, can be computed in $O(\iHy)$ time; this is easily observable given that $\mbf{Q}^{-1}$ is block-diagonal, $\mbf{F}^{-1}$ is lower block-bidiagonal, and the product $(\pricov^{-1} + \frac{1}{\sigma_\ihy^2}\mbf{1})^{-1} \mbf{v}$ can be computed in $O(\iHy)$ time using Cholesky decomposition, because $(\pricov^{-1} + \frac{1}{\sigma_\ihy^2}\mbf{1})$ is block-tridiagonal.
While the two terms in \eqref{eq:partial_marglike} can be combined into a single trace function, it is simpler to study their computational complexity separately.
Starting with the first term, we have
\begin{align}
& \frac{1}{2} (\mbf{y} - \primean)^T \mbf{P}_\ihy^{-1} \frac{\partial \mbf{P}_\ihy}{\partial \mbf{Q}_{C_{ij}}} \mbf{P}_\ihy^{-1} (\mbf{y} - \primean) \nonumber \\
\label{eq:partial_marg1}
& \quad\quad = \frac{1}{2} (\mbf{y} - \primean)^T \mbf{F}^{-T} \mbf{Q}_\ihy^{-1} \frac{\partial \mbf{Q}}{\partial \mbf{Q}_{C_{ij}}} \mbf{Q}_\ihy^{-1} \mbf{F}^{-1} (\mbf{y} - \primean),
\end{align}
where
\begin{align}
\frac{\partial \mbf{Q}}{\partial \mbf{Q}_{C_{ij}}} &= \diag\left( \mbf{0}, \frac{\partial \mbf{Q}_1}{\partial \mbf{Q}_{C_{ij}}}, \ldots, \frac{\partial \mbf{Q}_\iK}{\partial \mbf{Q}_{C_{ij}}} \right), \\
\frac{\partial \mbf{Q}_{n}}{\partial \mbf{Q}_{C_{ij}}}
&= \int_{t_{n-1}}^{t_n} \mbs{\Phi}(t_n,s) \mbf{L}(s) \frac{\partial \mbf{Q}_C}{\partial \mbf{Q}_{C_{ij}}} \mbf{L}(s)^T \mbs{\Phi}(t_n,s)^T \, ds, \nonumber \\
&= \int_{t_{n-1}}^{t_n} \mbs{\Phi}(t_n,s) \mbf{L}(s) \mbf{1}_{i,j} \mbf{L}(s)^T \mbs{\Phi}(t_n,s)^T \, ds,
\end{align}
and $\mbf{1}_{i,j}$ denotes a projection matrix with a $1$ at the $i$\textsuperscript{th} row and $j$\textsuperscript{th} column.
Taking advantage of the sparse matrices and previously mentioned fast matrix-vector products, it is clear that \eqref{eq:partial_marg1} can be computed in $O(\iHy)$ time (in contrast to the typical $O(\iHy^3)$ time).
\end{tolxemerg}

Taking a look at the second term, we have that
\begin{align}
&\frac{1}{2} \tr \left( \mbf{P}_\ihy^{-1}\frac{\partial \mbf{P}_\ihy}{\partial \mbf{Q}_{C_{ij}}} \right) \nonumber \\
& \quad\quad = \frac{1}{2} \tr \left( \left(\mbf{F}^{-T} \mbf{Q}_\ihy^{-1} \mbf{F}^{-1} \right) \left(\mbf{F} \frac{\partial \mbf{Q}}{\partial \mbf{Q}_{C_{ij}}} \mbf{F}^T \right) \right) \nonumber \\
& \quad\quad = \frac{1}{2} \tr \left( \frac{\partial \mbf{Q}}{\partial \mbf{Q}_{C_{ij}}} \mbf{Q}_\ihy^{-1} \right),
\end{align}
which can only be computed in $O(\iHy^2)$, due to the form of $\mbf{Q}_\ihy^{-1}$.
In general, the total complexity of training for this sparse class of prior is then bottlenecked at $O(\iHy^2)$. 
A complexity of $O(\iHy)$ can only be achieved by ignoring the additive measurement noise, $\sigma_\ihy^2 \mbf{1}$;
revisiting the second term of \eqref{eq:partial_marglike}, and setting $\mbf{P}_\ihy = \pricov$, we find that
\begin{equation}
\frac{1}{2} \tr \left( \pricov^{-1}\frac{\partial \pricov}{\partial \mbf{Q}_{C_{ij}}} \right) 
= \frac{1}{2} \tr \left( \frac{\partial \mbf{Q}}{\partial \mbf{Q}_{C_{ij}}} \mbf{Q}^{-1} \right),
\end{equation}
which can be computed in time $O(\iHy)$.
The effect of ignoring the measurement noise, $\sigma_\ihy^2 \mbf{1}$, is that the trained hyperparameters will result in an underconfident prior;
the degree of this underconfidence depends on the magnitude of the noise we are choosing to ignore.
If accurate ground-truth measurements are available and the size of the training dataset is very large, this approximation may be beneficial.

\subsection{Complexity}

\begin{tolxemerg}{}{2pt}
We conclude this section with a brief discussion of the time complexity of the overall algorithm when exploiting the sparse structure.  If we have $\iK$ measurement times and want to query the trajectory at $\iQry$ additional times of interest, the complexity of the resulting algorithm using GP regression with {\em any} linear (or nonlinear), time-varying process model driven by white noise will be $O(\iK+\iQry)$.  This is broken into the two major steps as follows.  The initial solution to find $\opmean$ (at the measurement times) can be done in $O(\iK)$ time (per iteration) owing to the block-tridiagonal structure discussed earlier.  Then, the cost of the queries at $\iQry$ other times of interest is $O(\iQry)$ since each individual query is $O(1)$.  Clearly, $O(\iK+\iQry)$ is a big improvement over the $O(\iK^3+\iK^2\iQry)$ cost when we did not exploit the sparse structure of the problem.  
\end{tolxemerg}

\section{Mobile Robot Example}
\label{sec:experiment}

\subsection{Linear Constant-Velocity GP Prior}
\label{sec:exp_ltipri}

\begin{tolxemerg}{}{2pt}
We will demonstrate the advantages of the sparse structure through an example employing the `constant-velocity' prior, $\ddot{\mbf{p}}(t) = \mbf{w}(t)$.  This can be expressed as a linear, time-invariant SDE of the form in~\eqref{eq:ltvsde} with
\begin{equation}
\mbf{x}(t) = \bbm \mbf{p}(t) \\ \dot{\mbf{p}}(t) \ebm, \; \sdemat(t) = \bbm \mbf{0} & \mbf{1} \\ \mbf{0} & \mbf{0} \ebm, \; \mbf{v}(t) = \mbf{0}, \; \mbf{L}(t) = \bbm \mbf{0} \\ \mbf{1} \ebm,
\end{equation}
where $\mbf{p}(t) = \bbm x(t) & y(t) & \theta(t) \ebm^T$ is the pose and $\dot{\mbf{p}}(t)$ is the pose rate.  In this case, the transition function is
\begin{equation}
\mbs{\Phi}(t,s) = \bbm \mbf{1} & (t-s) \mbf{1} \\ \mbf{0} & \mbf{1} \ebm,
\end{equation}
which can be used to construct $\sdemat$ (or $\sdemat^{-1}$ directly).  As we will be doing a STEAM example, we will constrain the first trajectory state to be $\mbf{x}(t_0) = \mbf{0}$ and so will have no need for $\primean_0$ and $\pricov_0$.  For $\ik=1\ldots \iK$, we have $\mbf{v}_\ik = \mbf{0}$ and
\begin{equation}
\mbf{Q}_\ik = \bbm \frac{1}{3} \Delta t_\ik^3 \mbf{Q}_C & \frac{1}{2} \Delta t_\ik^2 \mbf{Q}_C \\ \frac{1}{2} \Delta t_\ik^2 \mbf{Q}_C & \Delta t_\ik \mbf{Q}_C \ebm,
\end{equation}
with $\Delta t_\ik = t_\ik - t_{\ik-1}$.  The inverse blocks are
\begin{equation}
\mbf{Q}_\ik^{-1} = \bbm 12 \Delta t_\ik^{-3} \mbf{Q}_C^{-1} & -6 \Delta t_\ik^{-2} \mbf{Q}_C^{-1} \\ -6 \Delta t_\ik^{-2} \mbf{Q}_C^{-1} & 4 \Delta t_\ik^{-1} \mbf{Q}_C^{-1} \ebm,
\end{equation}
so we can build $\mbf{Q}^{-1}$ directly.  We now have everything we need to represent the prior: $\sdemat^{-1}$, $\mbf{Q}^{-1}$, and $\mbf{v}=\mbf{0}$, which can be used to construct $\pricov^{-1}$.
\end{tolxemerg}

We will also augment the trajectory with a set of $L$ landmarks, $\mbs{\ell}$, into a combined state, $\mbf{z}$, in order to consider the STEAM problem:
\begin{equation}
\mbf{z} = \bbm \mbf{x} \\ \mbs{\ell} \ebm, \quad \mbs{\ell} = \bbm \mbs{\ell}_1 \\ \vdots \\ \mbs{\ell}_L \ebm, \quad \mbs{\ell}_i = \bbm x_i \\ y_i \ebm.
\end{equation}
While others have folded velocity estimation into discrete-time, filter-based SLAM \citep{davison07} and even discrete-time, batch SLAM \citep{grau12}, we are actually proposing something more general than this:  the choice of prior tells us what to use for the trajectory states.   And, although we solve for the state at a discrete number of measurement times, our setup is based on an underlying continuous-time prior, meaning that we can query it at any time of interest in a principled way.  

\subsection{Nonlinear Constant-Velocity GP Prior}
\label{sec:exp_ntvpri}

\begin{tolxemerg}{}{2pt}
Since the main source of acceleration in this example is robot-oriented (the actuated wheels), we investigate the use of an alternate `constant-velocity' prior, with the white noise affecting acceleration in the robot body frame, $\mbsdot{\nu}(t) = \mbf{w}(t)$. 
This \emph{nonlinear} prior can be written as,
\begin{align}
\mbfdot{x}(t) &= \mbf{f}(\mbf{x}(t),\mbf{u}(t),\mbf{w}(t)) \nonumber \\
&= \bbm \mbf{0} & \mbf{R}_{IB}(\theta(t)) \\ \mbf{0} & \mbf{0} \ebm \mbf{x}(t) + \mbf{u}(t) + \bbm \mbf{0} \\ \mbf{1} \ebm \mbf{w}(t),
\end{align}
where
\begin{equation}
\mbf{R}_{IB}(\theta(t)) = \bbm \cos\theta(t) & -\sin\theta(t) & 0 \\ 
\sin\theta(t) &  \cos\theta(t) & 0 \\ 
0      &        0       & 1 \ebm
\end{equation}
is a rotation matrix between the inertial and robot body frame, and
the \emph{Markovian} state is $\mbf{x}(t) = \bbm \mbf{p}(t)^T \; \mbs{\nu}(t)^T \ebm^T$, where $\mbs{\nu}(t) = \bbm v(t) & u(t) & \omega(t) \ebm^T = \mbf{R}_{IB}(\theta(t))^T \mbfdot{p}(t)$ is the robot-oriented velocity (with longitudinal, latitudinal, and rotational components, respectively).
\end{tolxemerg}

Linearizing about an arbitrary operating point, $\opmean(t)$, the components $\mbf{F}(t)$, $\mbf{v}(t)$ and $\mbf{L}(t)$ from $\eqref{eq:ntvsde_lin}$ are straight-forward to derive.
Similarly, to the linear prior example described above, we will define the exogenous input $\mbf{u}(t) = \mbf{0}$; however, we note that $\mbf{v}(t) \neq \mbf{0}$ (see \eqref{eq:ntvsde_vlin}). 
Since expressions for $\mbs{\Phi}(t,s)$ and $\mbf{Q}_\ik$ are not obvious, we will rely on numerical integration for their evaluation.

\subsection{Measurement Models}

We will use two types of measurements:  range/bearing to landmarks (using a laser rangefinder) and wheel odometry (in the form of robot-oriented velocity).   The range/bearing measurement model takes the form
\begin{eqnarray}
\mbf{y}_{\ik i} & = & \mbf{g}_{\rm rb}( \mbf{x}(t_\ik), \mbs{\ell}_i ) + \mbf{n}_{\ik i} \nonumber \\
& = & \bbm \sqrt{(x_i -x(t_\ik) )^2 + (y_i-y(t_\ik))^2} \, \\ \mbox{atan2}(y_i-y(t_\ik),x_i-x(t_\ik)) \, \ebm + \mbf{n}_{\ik i}.
\end{eqnarray}
The wheel odometry measurement model gives the longitudinal and rotational speeds of the robot, taking the form
\begin{equation}
\mbf{y}_\ik = \mbf{g}_{\rm wo}( \mbf{x}(t_\ik) ) + \mbf{n}_\ik = \bbm \cos\theta(t_\ik) & \sin\theta(t_\ik) & 0 \\ 0 & 0 & 1 \ebm \hspace{-0.08cm} \mbfdot{p}(t_\ik) + \mbf{n}_\ik,
\end{equation}
for the LTI SDE prior, and
\begin{equation}
\mbf{y}_\ik = \mbf{g}_{\rm wo}( \mbf{x}(t_\ik) ) + \mbf{n}_\ik = \bbm 1 & 0 & 0 \\ 0 & 0 & 1 \ebm \mbs{\nu}(t_\ik) + \mbf{n}_\ik,
\end{equation}
for the NTV SDE prior.  
Note that in both cases the velocity information is extracted easily from the state since we are estimating it directly.
The Jacobians with respect to the state, $\mbf{x}(t)$, are straightforward to derive.

\subsection{Exploiting Sparsity}

\begin{figure}
	\includegraphics[width=\columnwidth]{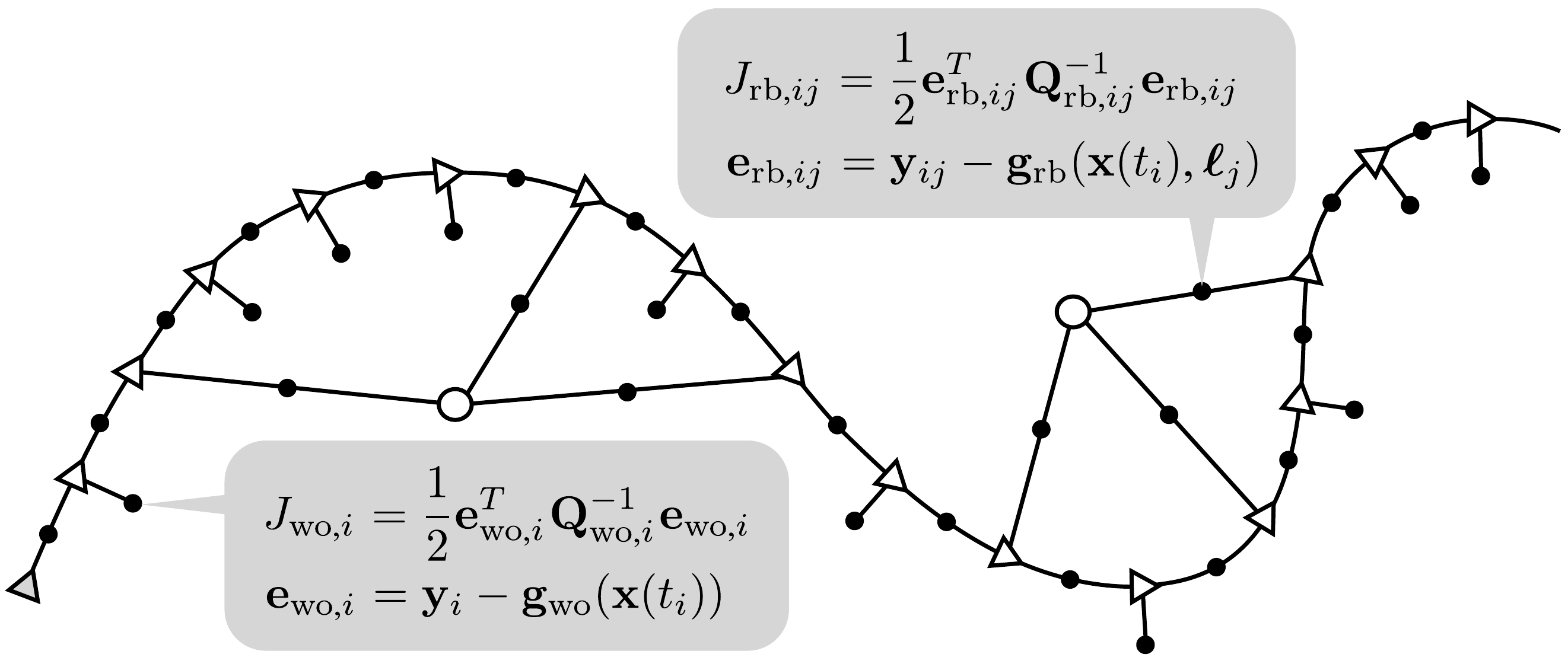}
	\caption{Factor-graph representation of our STEAM problem.  There are factors (black dots) for (i) the prior (binary), (ii) the landmark measurements (binary), and (iii) the wheel odometry measurements (unary).  Triangles are {\em trajectory states} (position and velocity, for this prior); the first trajectory state is locked.  Hollow circles are landmarks.}
	\label{fig:steam}
\end{figure}

\begin{tolxemerg}{}{4pt}
Figure~\ref{fig:steam} shows an illustration of the STEAM problem we are considering.   In terms of linear algebra, at each iteration we need to solve a linear system of the form
\begin{equation}
\label{eq:steam}
\underbrace{\bbm \mbf{W}_{xx} & \mbf{W}_{\ell x}^T \\ \mbf{W}_{\ell x} & \mbf{W}_{\ell\ell} \ebm}_{\mbf{W}} \underbrace{\bbm \delta\mbf{x}^\star \\ \delta\mbs{\ell}^\star \ebm}_{\delta\mbf{z}^\star} = \underbrace{\bbm \mbf{b}_x \\ \mbf{b}_{\ell} \ebm}_{\mbf{b}},
\end{equation}
which retains exploitable structure despite introducing landmarks to the state \citep{brown58}. In particular, $\mbf{W}_{xx}$ is block-tridiagonal (due to our GP prior) and $\mbf{W}_{\ell\ell}$ is block-diagonal; the sparsity of the off-diagonal block, $\mbf{W}_{\ell x}$, depends on the specific landmark observations.  We reiterate the fact that if we marginalize out the $\dot{\mbf{p}}(t_\ik)$ (or $\mbs{\nu}(t_\ik)$) variables and keep only the $\mbf{p}(t_\ik)$ variables to represent the trajectory, the $\mbf{W}_{xx}$ block becomes dense (for these priors); this is precisely the approach of \citet{tong_ijrr13b}.
\end{tolxemerg}

To solve~\eqref{eq:steam} efficiently, we can begin by either exploiting the sparsity of $\mbf{W}_{xx}$ or of $\mbf{W}_{\ell\ell}$.  Since each trajectory variable represents a unique measurement time (range/bearing or odometry), there are potentially a lot more trajectory variables than landmark variables, $L \ll M$,  so we will exploit $\mbf{W}_{xx}$.  

We use a sparse (lower-upper) Cholesky decomposition:
\begin{equation}
\underbrace{\bbm \mbf{V}_{xx} & \mbf{0} \\ \mbf{V}_{\ell x} & \mbf{V}_{\ell\ell} \ebm}_{\mbf{V}} \underbrace{\bbm \mbf{V}_{xx}^T & \mbf{V}_{\ell x}^T \\ \mbf{0} & \mbf{V}_{\ell\ell}^T \ebm}_{\mbf{V}^T} = \underbrace{\bbm \mbf{W}_{xx} & \mbf{W}_{\ell x}^T \\ \mbf{W}_{\ell x} & \mbf{W}_{\ell\ell} \ebm}_{\mbf{W}} 
\end{equation}
We first decompose $\mbf{V}_{xx}\mbf{V}_{xx}^T = \mbf{W}_{xx}$, which can be done in $O(\iK)$ time owing to the block-tridiagonal sparsity.  The resulting $\mbf{V}_{xx}$ will have only the main block-diagonal and the one below it non-zero.  We can then solve $\mbf{V}_{\ell x} \mbf{V}_{xx}^T = \mbf{W}_{\ell x}$ for $\mbf{V}_{\ell x}$ in $O(L \iK)$ time.  Finally, we decompose $\mbf{V}_{\ell\ell} \mbf{V}_{\ell\ell}^T = \mbf{W}_{\ell\ell} - \mbf{V}_{\ell x} \mbf{V}_{\ell x}^T$, which we can do in $O(L^3 + L^2\iK)$ time.  This completes the decomposition in $O(L^3 +L^2\iK)$ time.
We then perform the standard forward-backward passes, ensuring to exploit the sparsity:  first solve $\mbf{V} \mbf{d} = \mbf{b}$ for $\mbf{d}$, then $\mbf{V}^T \, \delta\mbf{z}^\star = \mbf{d}$ for $\delta\mbf{z}^\star$, both in $O(L^2 +L\iK)$ time.   Note, this approach does not marginalize out any variables during the solve, as this can ruin the sparsity (i.e., we avoid inverting $\mbf{W}_{xx}$).  The whole solve is $O(L^3 +L^2\iK)$.

At each iteration, we update the state, $\bar{\mbf{z}} \leftarrow \bar{\mbf{z}} + \delta\mbf{z}^\star$, and iterate to convergence.  Finally, we query the trajectory at $\iQry$ other times of interest using the GP interpolation discussed earlier.  The whole procedure is then $O(L^3 +L^2\iK + \iQry)$, including the extra queries.

Due to the addition of landmarks, the cost of a STEAM problem must be either be of order $O(L^3 + L^2\iK + \iQry)$ or $O(\iK^3 + \iK^2L + \iQry)$, depending on the way we choose to exploit $\mbf{W}_{xx}$ or $\mbf{W}_{\ell\ell}$.
The state reduction scheme presented in Section~\ref{sec:measinterp} becomes very attractive when both the number of measurements and landmarks are very high;
estimating the state at $\iKf$ \emph{keytimes}, $\iKf < \iK$ , and exploiting $\mbf{W}_{\ell\ell}$, the procedure becomes order $O(\iKf^3 + \iKf^2L + \iK + \iQry)$, which is the same complexity as a traditional discrete-time SLAM problem (with the addition of the $\iQry$ query times).

\subsection{Experiment}
\label{sec:experiment_subsec}

\begin{figure*}
	\subfigure[{\em GP-Traj-Sparse-LTI} trajectory sample.]{
		\includegraphics[width=0.49\textwidth]{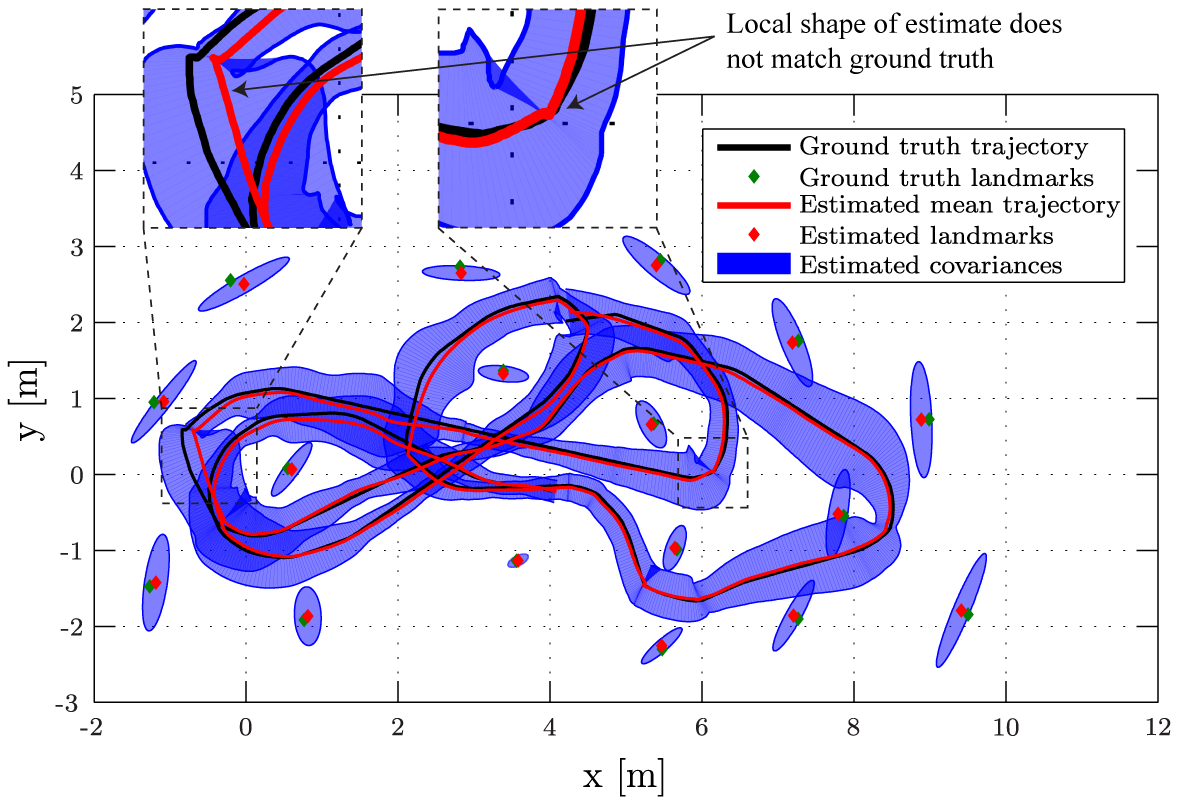}
		\label{fig:trajplot_linear}}
	\subfigure[{\em GP-Traj-Sparse-NTV} trajectory sample.]{
		\includegraphics[width=0.49\textwidth]{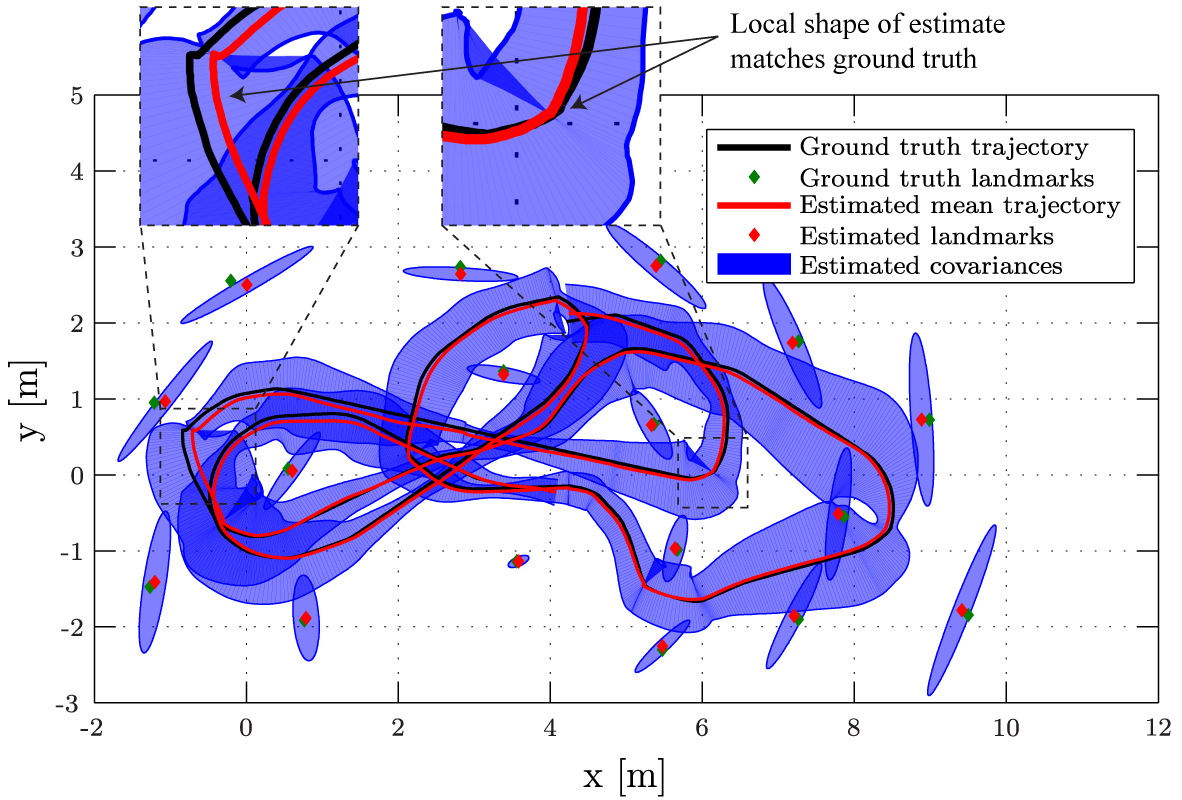}
		\label{fig:trajplot_nonlinear}}
	\caption{The smooth and continuous trajectories and $3\sigma$ covariance envelope estimates produced by the {\em GP-Traj-Sparse} estimators (both linear and nonlinear) over a short segment of the dataset.} 	
	\label{fig:trajectory_plot}
\end{figure*}

For experimental validation, we employed the same mobile robot dataset as used by \citet{tong_ijrr13b}.  This dataset consists of a mobile robot equipped with a laser rangefinder driving in an indoor, planar environment amongst a forest of $17$ plastic-tube landmarks.  The odometry and landmark measurements are provided at a rate of 1Hz, and additional trajectory queries are computed at a rate of 10Hz after estimator convergence.  Ground-truth for the robot trajectory and landmark positions is provided by a Vicon motion capture system.

\begin{tolxemerg}{}{2pt}
We implemented three estimators for comparison.  The first was the algorithm described by \citet{tong_ijrr13b}, {\em GP-Pose-Dense}, the second was a naive version of our estimator, {\em GP-Traj-Dense}, based on the LTI SDE prior described in Section~\ref{sec:exp_ltipri}, but did not exploit sparsity, and the third was a full implementations of our estimator, {\em GP-Traj-Sparse}, that exploited the sparsity structure as described in this paper. The final estimator has two variants, {\em GP-Traj-Sparse-LTI} and {\em GP-Traj-Sparse-NTV}, based on the LTI SDE and NTV SDE priors described in Sections~\ref{sec:exp_ltipri} and \ref{sec:exp_ntvpri}, respectively;
as the first two estimators, {\em GP-Pose-Dense} and {\em GP-Traj-Dense}, are only of interest with regard to computational performance, they each implement only the LTI `constant-velocity' prior.
\end{tolxemerg}

For this experiment, we obtained $\mbf{Q}_C$ for both the LTI and NTV priors by modelling it as a diagonal matrix and taking the data-driven training approach using log marginal likelihood (with ground-truth measurements) described in Section~\ref{sec:hyper_training}.

\begin{tolxemerg}{300}{4pt}
Though the focus of the exactly sparse Gaussian process priors is to demonstrate the significant reductions in computational cost, we provide Figure~\ref{fig:trajectory_plot} to illustrate the smooth trajectory estimates we obtained from the continuous-time formulation.  While the algorithms differed in their number of degrees of freedom and types of their estimated states, their overall accuracies were similar for this dataset.
The {\em GP-Traj-Sparse-NTV} algorithm differed slightly from the others; qualitatively, we found that the trajectory estimated by the {\em GP-Traj-Sparse-NTV} algorithm more accurately matched the local shape of the ground-truth on many occasions, such as the ones highlighted by the insets in Figure~\ref{fig:trajectory_plot}. Also, it is clear from the plotted $3\sigma$ covariance envelope that the estimate from the {\em GP-Traj-Sparse-NTV} algorithm tends to be more uncertain.
\end{tolxemerg}

\subsubsection{Computational Cost}

\begin{figure*}
	\subfigure[Kernel matrix construction time.]{
		\includegraphics[width=0.49\textwidth]{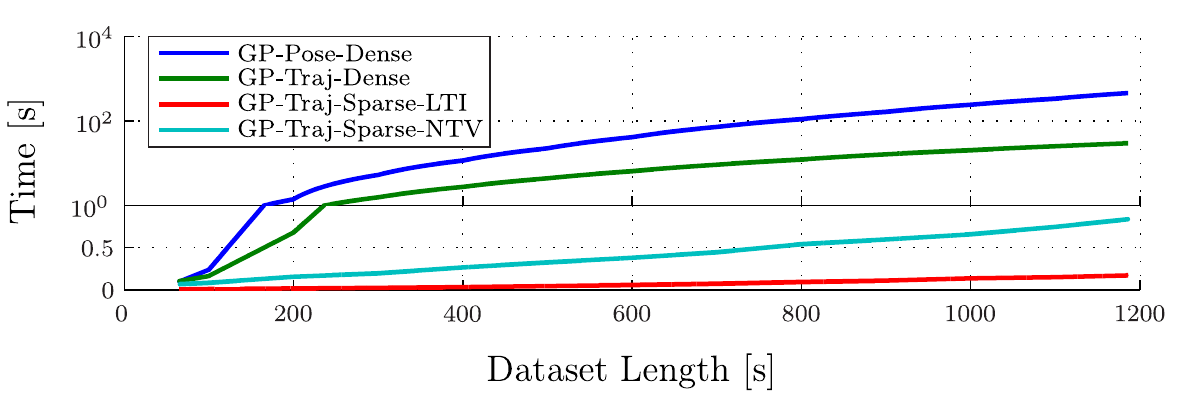}
		\label{fig:timing_kernel}}
	\subfigure[Optimization time per iteration.]{
		\includegraphics[width=0.49\textwidth]{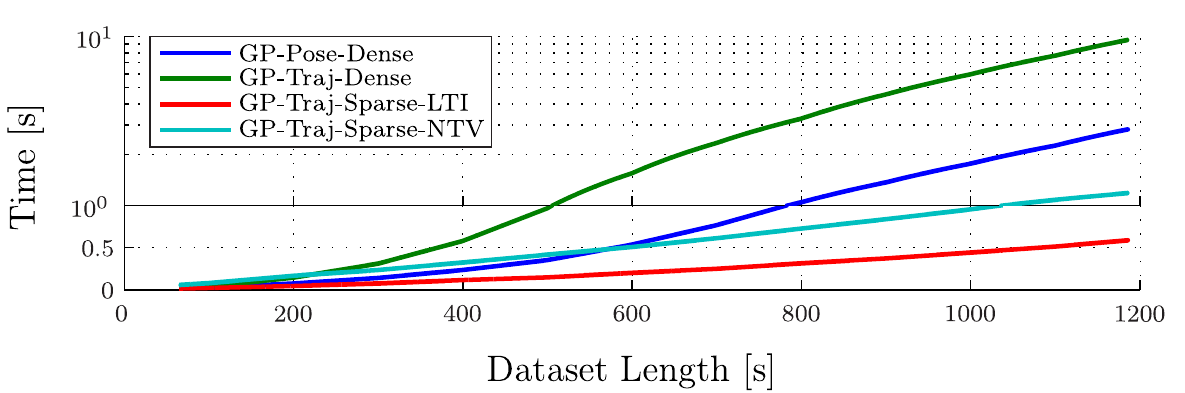}
		\label{fig:timing_optimization}}
	\subfigure[Interpolation time per additional query time.]{
		\includegraphics[width=0.49\textwidth]{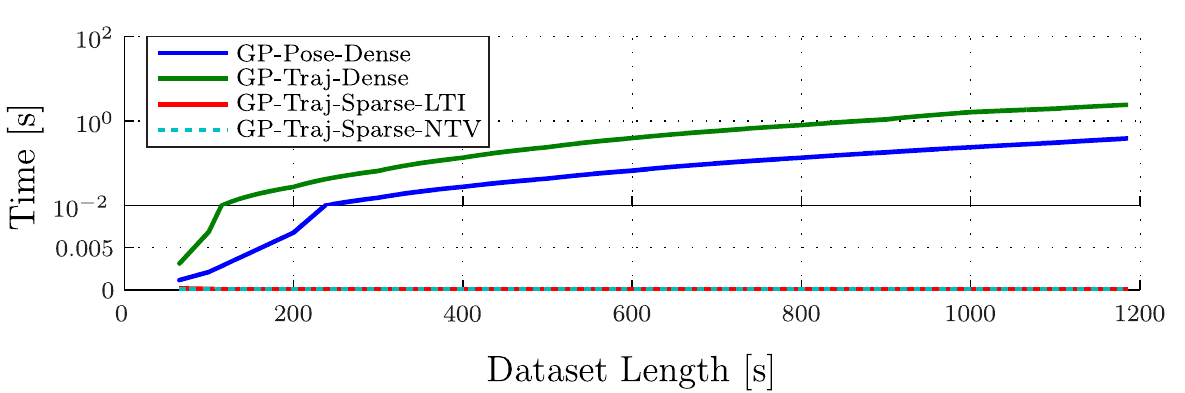}
		\label{fig:timing_interpolation}}
	\subfigure[Total computation time.]{
		\includegraphics[width=0.49\textwidth]{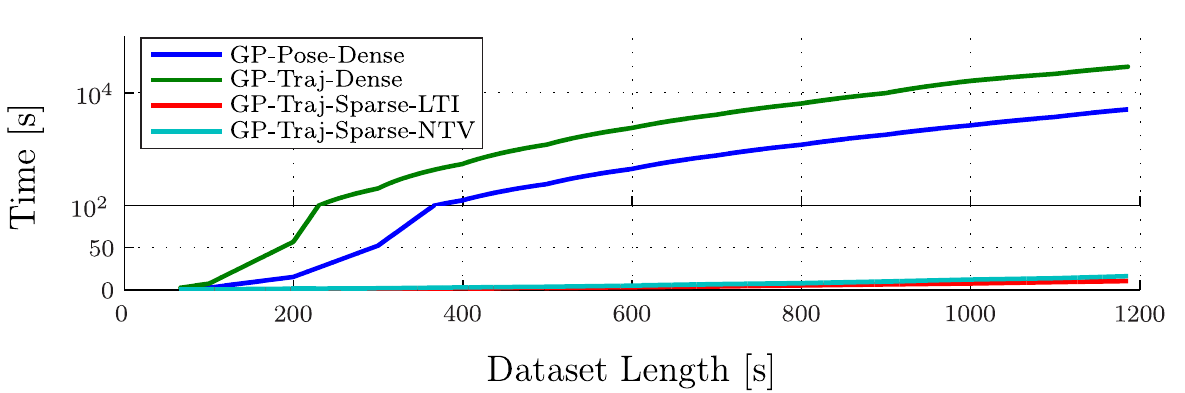}
		\label{fig:timing_total}}
	\caption{Plots comparing the compute time (as a function of trajectory length) for the {\em GP-Pose-Dense} algorithm described by \citet{tong_ijrr13b} and three versions of our approach: {\em GP-Traj-Dense} (does not exploit sparsity) and the two {\em GP-Traj-Sparse} variants (exploit sparsity).  The plots confirm the predicted computational complexities of the various methods; notably, the {\em GP-Traj-Sparse} estimators have linear cost in trajectory length.  Please note the change from a linear to a log scale in the upper part of each plot.}
	\label{fig:timing_plots}
\end{figure*}

\begin{tolxemerg}{}{2pt}
To evaluate the computational savings of exploiting an exactly sparse GP prior, we implemented all algorithms in Matlab with a 2.4GHz i7 processor and 8GB of 1333MHz DDR3 RAM and timed the computation for segments of the dataset of varying lengths.  These results are shown in Figure~\ref{fig:timing_plots}, where we provide the computation time for the individual operations that benefit most from the sparse structure, as well as the overall processing time.  
\end{tolxemerg}

We see that the {\em GP-Traj-Dense} algorithm is much slower than the original {\em GP-Pose-Dense} algorithm of \citet{tong_ijrr13b}. This is because we have reintroduced the velocity part of the state, thereby doubling the number of variables associated with the trajectory.   However, once we start exploiting the sparsity with the {\em GP-Traj-Sparse} methods, the increase in number of variables pays off.

For the {\em GP-Traj-Sparse} methods,  we see in Figure~\ref{fig:timing_kernel} that the kernel matrix construction was linear in the number of estimated states.  This can be attributed to the fact that we constructed the sparse $\pricov^{-1}$ directly.  As predicted, the optimization time per iteration was also linear in Figure~\ref{fig:timing_optimization}, and the interpolation time per additional query was constant regardless of state size in Figure~\ref{fig:timing_interpolation}.  Finally, Figure~\ref{fig:timing_total} shows that the total compute time was also linear.

The additional cost of the {\em GP-Traj-Sparse-NTV} algorithm over the {\em GP-Traj-Sparse-LTI} algorithm in kernel construction time is due to the linearization and numerical integration of the prior mean and covariance. The optimization time of the {\em GP-Traj-Sparse-NTV} algorithm is also affected because the kernel matrix must be reconstructed from a new linearization of the prior during every optimization iteration. The {\em GP-Traj-Sparse-NTV} algorithm also incurs some numerical integration cost in interpolation time, but it is constant and very small.

We also note that the number of iterations for optimization convergence varied for each algorithm.  In particular, we found that the {\em GP-Traj-Sparse}  implementations converged in fewer iterations than the other implementations due to the fact that we constructed the inverse kernel matrix directly, which resulted in greater numerical stability.  The {\em GP-Traj-Sparse}  approaches clearly outperform the other algorithms in terms of computational cost.

\subsubsection{Increasing Nonlinearity}

\begin{figure*}
	\subfigure[RMS errors using odometry measurements.]{
		\includegraphics[width=0.48\textwidth]{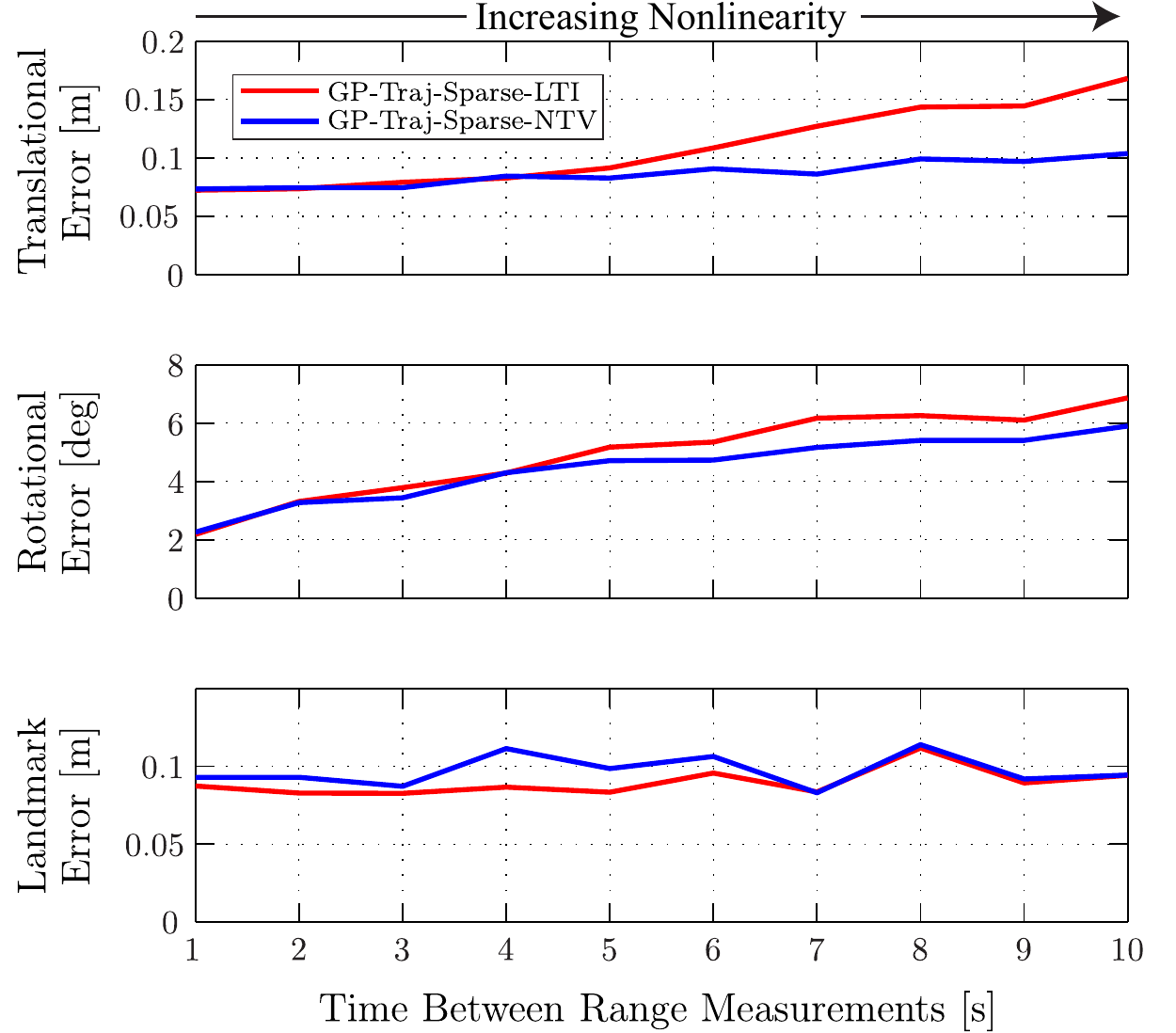}
		\label{fig:error_with_odom}}
	\hfill
	\subfigure[RMS errors without using odometry measurements.]{
		\includegraphics[width=0.48\textwidth]{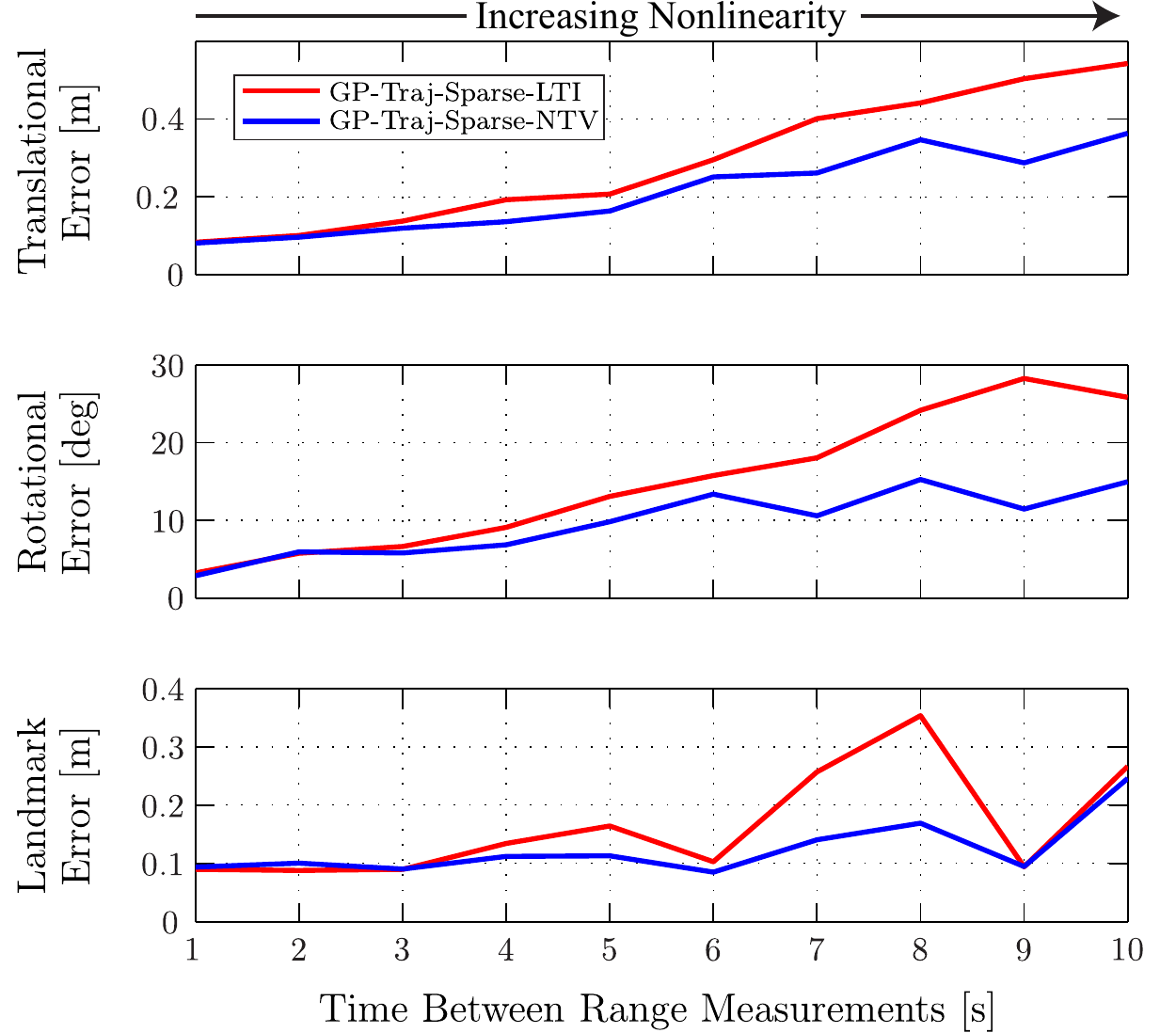}
		\label{fig:error_without_odom}}
	\caption{Plots comparing use of the {\em GP-Traj-Sparse-LTI} and {\em -NTV} algorithms for an increasingly nonlinear problem; the dataset was made more nonlinear by varying the interval between available range measurements. The results in \protect\subref{fig:error_with_odom} used the odometry measurements available at 1\,Hz, while \protect\subref{fig:error_without_odom} was made to be even more nonlinear by excluding the use of odometry measurements. The plots show that for a small interval between range measurements, both estimators perform similarly; as the interval was increased, the estimate provided by the nonlinear prior is consistently better in both translation and angular error. Notably, when odometry is available, both estimators are able to achieve a similar orientation performance; the negative effect of removing odometry measurements is more prominent on the linear prior estimator.  }
	\label{fig:nonlinearity_errors}
\end{figure*}

\begin{figure*}
	\subfigure[With odometry measurements.]{
		\includegraphics[width=0.49\textwidth]{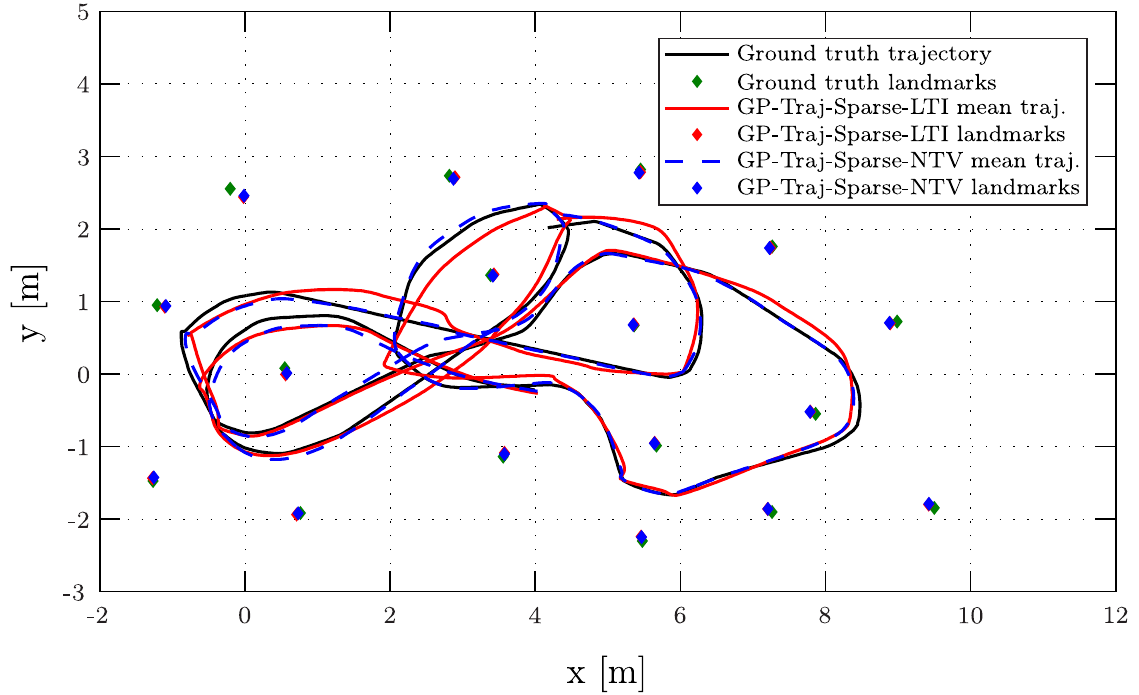}
		\label{fig:trajplot_with_odom}}
	\hfill
	\subfigure[Without odometry measurements.]{
		\includegraphics[width=0.49\textwidth]{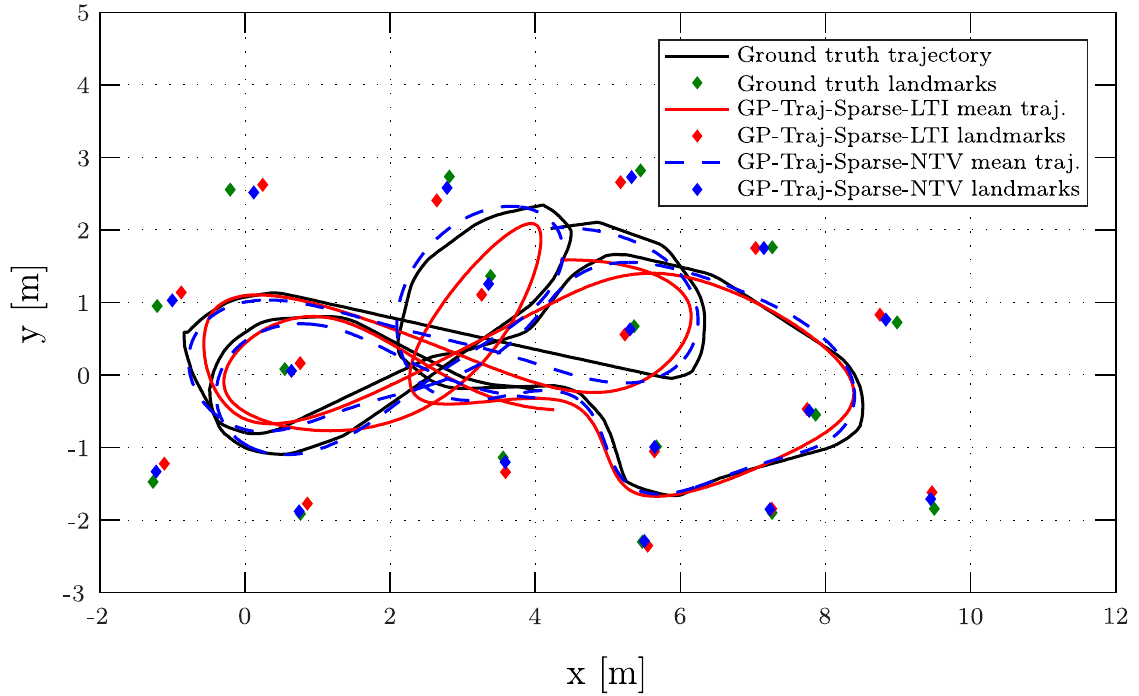}
		\label{fig:trajplot_without_odom}}
	\caption{Plots showing the {\em GP-Traj-Sparse-LTI} and {\em -NTV} estimates for the same small trajectory subsection as Figure~\ref{fig:trajectory_plot}, with an interval between range measurements of 7 seconds. Results in \protect\subref{fig:trajplot_with_odom} used odometry measurements, while \protect\subref{fig:trajplot_without_odom} did not. }
	\label{fig:nonlinearity_trajplot}
\end{figure*}

In a problem with fairly accurate and high-rate measurements, both the {\em GP-Traj-Sparse-LTI} and {\em GP-Traj-Sparse-NTV} estimators provide similar accuracy. 
In order to expose the benefit of a nonlinear prior based on the expected motion of the vehicle, we increase the nonlinearity of the problem by reducing measurement frequency. 

\begin{tolxemerg}{}{4pt}
The result of varying range measurement frequency, with and without the use of odometry measurements, is shown in Figure~\ref{fig:nonlinearity_errors}.
In general, it is clear that as the interval between range measurements is increased, the {\em GP-Traj-Sparse-NTV} estimator is able to produce a more accurate estimate of the continuous-time pose (translation and rotation) than the {\em GP-Traj-Sparse-LTI} estimator.
\end{tolxemerg}

In the case that the 1\,Hz odometry measurements are available, as seen in Figure~\ref{fig:error_with_odom}, the difference in the rotation estimates is small, because the {\em GP-Traj-Sparse-LTI} estimator has a good amount of information about its heading; however, in the case that the odometry measurements are unavailable, as seen in Figure~\ref{fig:error_without_odom}, the advantage of the nonlinear prior implemented by the {\em GP-Traj-Sparse-NTV} estimator is prominent with respect to the rotational estimate.

In order to gain some qualitative intuition about how the continuous-time pose estimates are affected by the reduction of measurements, Figure~\ref{fig:nonlinearity_trajplot} shows the trajectories for the same small subsection as presented in Figure~\ref{fig:trajectory_plot}; the estimates used an interval between range measurements of 7 seconds and are shown with and without the use of odometry measurements.
In both plots, it is clear that the {\em GP-Traj-Sparse-NTV} estimator matches the ground-truth more closely, as previously indicated by the error plots in Figure~\ref{fig:nonlinearity_errors}.

\section{Discussion and Future Work}
\label{sec:discussion}

It is worth elaborating on a few issues. The main reason that the $\mbf{W}_{xx}$ block is sparse in our approach, as compared to \citet{tong_ijrr13b}, is that we reintroduced velocity variables that had effectively been marginalized out.  This idea of reintroducing variables to regain exact sparsity has been used before by \citet{eustice06} in the delayed state filter and by \citet{walter07} in the extended information filter.  This is a good lesson to heed:  the underlying structure of a problem may be exactly sparse, but by marginalizing out variables it appears dense.  For us this means we need to use a Markovian trajectory state that is appropriate to our prior.

In much of mobile robotics, odometry measurements are treated more like inputs to the mean of the prior than pure measurements.  We believe this is a confusing thing to do as it conflates two sources of uncertainty:  the prior over trajectories and the odometry measurement noise.  In our framework, we have deliberately separated these two functions and believe this is easier to work with and understand.  We can see these two functions directly in Figure~\ref{fig:steam}, where the prior is made up of binary factors joining consecutive trajectory states, and odometry measurements are unary factors attached to some of the trajectory states (we could have used binary odometry factors but chose to set things up this way due to the fact that we were explicitly estimating velocity).  

While our analysis appears to be restricted to a small class of covariance functions, we have only framed our discussions in the context of robotics.  Recent developments from machine learning \citep{hartikainen10} and signal processing \citep{sarkka13} have shown that it is possible to generate other well-known covariance functions using a LTV SDE (some exactly and some approximately).  This means they can be used with our framework.  One example is the Mat\'{e}rn covariance family \citep{rasmussen06}, 
\begin{equation} %
\pricov_{\rm m}(t,t^\prime) = \sigma^2 \frac{2^{1-\nu}}{\Gamma(\nu)}\left( \frac{\sqrt{2\nu}}{\ell} |t-t^\prime|\right)^\nu \hspace{-0.1cm} K_\nu \left(  \frac{\sqrt{2\nu}}{\ell} |t-t^\prime| \right) \hspace{-0.05cm} \mbf{1}
\end{equation} %
where $\sigma$, $\nu$, $\ell >0$ are magnitude, smoothness, and length-scale parameters, $\Gamma(\cdot)$ is the gamma function, and $K_\nu(\cdot)$ is the modified Bessel function.  For example, if we let
\begin{equation}
\mbf{x}(t) = \bbm \mbf{p}(t) \\ \dot{\mbf{p}}(t) \ebm,
\end{equation}
with $\nu = p + \frac{1}{2}$ with $p=1$ and use the following SDE:
\begin{equation}
\dot{\mbf{x}}(t) = \bbm \mbf{0} & \mbf{1} \\  -\lambda^2 \mbf{1} & -2\lambda \mbf{1}  \ebm \mbf{x}(t) + \bbm \mbf{0} \\ \mbf{1} \ebm \mbf{w}(t),
\end{equation}
where $\lambda = \sqrt{2\nu}/\ell$ and $\mbf{w}(t) \sim \mathcal{GP}\left( \mbf{0}, \mbf{Q}_C \, \delta(t-t^\prime) \right)$ (our usual white noise) with power spectral density matrix,
\begin{equation}
\mbf{Q}_C = \frac{2\sigma^2 \pi^{\frac{1}{2}} \lambda^{2p+1} \Gamma(p+1)}{\Gamma(p+\frac{1}{2})} \mbf{1},
\end{equation}
then we have that $\mbf{p}(t)$ is distributed according to the Mat\'{e}rn covariance family: $\mbf{p}(t) \sim \mathcal{GP}(\mbf{0}, \pricov_{\rm m}(t,t^\prime) )$ with $p=1$.  Another way to look at this is that passing white noise through LTV SDEs produces particular coloured-noise priors (i.e., not flat across all frequencies).   

In terms of future work, we plan to incorporate the dynamics (i.e., kinematics plus Newtonian mechanics) of a robot platform into the GP priors;  real sensors do not move arbitrarily through the world as they are usually attached to massive robots and this serves to constrain the motion.  Another idea is to incorporate {\em latent force models} into our GP priors (e.g., see \citet{alvarez09} or \citet{hartikainen12}).  We also plan to look further at the sparsity of STEAM and integrate our work with modern solvers to tackle large-scale problems; this should allow us to exploit more than just the primary sparsity of the problem and do so in an online manner.

\section{Conclusion}
\label{sec:conclusion}

\begin{tolxemerg}{}{4pt}
We have considered continuous-discrete estimation problems where a trajectory is viewed as a one-dimensional {\em Gaussian process} (GP), with time as the independent variable and measurements acquired at discrete times.  Querying the trajectory can be viewed as nonlinear, GP regression.   Our main contribution in this paper is to show that this querying can be accomplished very efficiently.  To do this, we exploited the Markov property of our GP priors (generated by nonlinear, time-varying stochastic differential equations driven by white noise) to construct an inverse kernel matrix that is sparse.  This makes it fast to solve for the state at the measurement times (as is commonly done in vision and robotics) but also at any other time(s) of interest through GP interpolation.  Other implications of this sparsity were discussed with respect to hyperparameter training, and including measurements at query times.  We also considered a slight generalization of the SLAM problem, {\em simultaneous trajectory estimation and mapping} (STEAM), which makes use of a continuous-time trajectory prior and allows us to query the state at any time of interest in an efficient manner. We hope this paper serves to deepen the connection between classical state estimation theory and recent machine learning methods by viewing batch estimation through the lens of Gaussian process regression.
\end{tolxemerg}

\begin{acknowledgements}
Thanks to Dr.\ Alastair Harrison at Oxford who asked the all-important question:  {\em how can the GP estimation approach \citep{tong_ijrr13b} be related to factor graphs?}  This work was supported by the Canada Research Chair Program, the Natural Sciences and Engineering Research Council of Canada, and the Academy of Finland.
\end{acknowledgements}

% BibTeX users please use one of
\bibliographystyle{apalike} 
%\bibliography{refs} 

% Hardcoded bibliography for latex servers
\small{

}

\end{document}